\documentclass[letterpaper]{article} 
\usepackage{aaai24}  
\usepackage{times}  
\usepackage{helvet}  
\usepackage{courier}  
\usepackage[hyphens]{url}  
\usepackage{graphicx} 
\usepackage{natbib}  
\usepackage{caption} 
\newcommand{\name}{{\textrm{SACL}}}
\usepackage{algorithmic}

\usepackage{amsmath,amsfonts,bm}
\usepackage{thm-restate}
\usepackage{amsthm}

\newtheorem{definition}{Definition}

\def\eqref#1{equation~\ref{#1}}

\def\1{\bm{1}}

\DeclareMathAlphabet{\mathsfit}{\encodingdefault}{\sfdefault}{m}{sl}
\SetMathAlphabet{\mathsfit}{bold}{\encodingdefault}{\sfdefault}{bx}{n}

\newcommand{\Var}{\mathrm{Var}}

\DeclareMathOperator*{\argmax}{arg\,max}

\usepackage{amsmath}
\usepackage{amssymb}
\usepackage{mathtools}
\usepackage{amsthm}
\usepackage{caption}
\usepackage{subfigure}
\usepackage{tabularx}
\usepackage{booktabs}
\usepackage[linesnumbered,ruled,vlined]{algorithm2e}
\usepackage{xcolor}         

\usepackage{newfloat}
\usepackage{listings}
\DeclareCaptionStyle{ruled}{labelfont=normalfont,labelsep=colon,strut=off} 
\title{Accelerate Multi-Agent Reinforcement Learning in Zero-Sum Games with Subgame Curriculum Learning}
\author {
    Jiayu Chen\textsuperscript{\rm 1}\equalcontrib,  
    Zelai Xu\textsuperscript{\rm 1}\equalcontrib,  
    Yunfei Li\textsuperscript{\rm 1},  
    Chao Yu\textsuperscript{\rm 1}, 
    \\
    Jiaming Song \textsuperscript{\rm 3},
    Huazhong Yang \textsuperscript{\rm 1},
    Fei Fang \textsuperscript{\rm 4},
    Yu Wang \textsuperscript{\rm 1$\dag$}, 
    Yi Wu  \textsuperscript{\rm 1, \rm 2}\thanks{Corresponding Authors.}
}
\affiliations {
    \textsuperscript{\rm 1}Tsinghua University, 
    \textsuperscript{\rm 2}Shanghai Qi Zhi Institute,
    \textsuperscript{\rm 3}Luma AI,
    \textsuperscript{\rm 4}Carnegie Mellon University\\
    \{jia768167535, zelai.eecs, jxwuyi\}@gmail.com
}

\usepackage{bibentry}

\begin{document}

\maketitle

\begin{abstract}
Learning Nash equilibrium (NE) in complex zero-sum games with multi-agent reinforcement learning (MARL) can be extremely computationally expensive. Curriculum learning is an effective way to accelerate learning, but an under-explored dimension for generating a curriculum is the difficulty-to-learn of the \emph{subgames} -- games induced by starting from a specific state. In this work, we present a novel subgame curriculum learning framework for zero-sum games.  
It adopts an adaptive initial state distribution by resetting agents to some previously visited states where they can quickly learn to improve performance.
Building upon this framework, we derive a subgame selection metric that approximates the squared distance to NE values and further adopt a particle-based state sampler for subgame generation. Integrating these techniques leads to our new algorithm, \emph{\underline{S}ubgame \underline{A}utomatic \underline{C}urriculum \underline{L}earning} ({\name}), which is a realization of the subgame curriculum learning framework. {\name} can be combined with any MARL algorithm such as MAPPO. Experiments in the
particle-world environment and Google Research Football environment show {\name} produces much stronger policies than baselines. In the challenging hide-and-seek quadrant environment, {\name} produces all four emergent stages and uses only half the samples of MAPPO with self-play. The project website is at \url{https://sites.google.com/view/sacl-rl}. 

\end{abstract}


\section{Introduction}\label{sec:intro}
Applying reinforcement learning (RL) to zero-sum games has led to enormous success, with trained agents defeating professional humans in Go~\citep{silver2016mastering}, StarCraft II~\citep{vinyals2019grandmaster}, and Dota 2~\citep{berner2019dota}.
To find an approximate Nash equilibrium (NE) in complex games, these works often require a tremendous amount of training resources including hundreds of GPUs and weeks or even months of time.
The unaffordable cost prevents RL from more real-world applications beyond these flagship projects supported by big companies and makes it important to develop algorithms that can learn close-to-equilibrium strategies in a substantially more efficient manner.

One way to accelerate training is curriculum learning -- training agents in tasks from easy to hard. Many existing works in solving zero-sum games with MARL generate a curriculum by choosing whom to play with. They often use self-play to provide a natural policy curriculum as the agents are trained against increasingly stronger opponents~\citep{bansal2018emergent,baker2020emergent}.
The self-play framework can be further extended to population-based training (PBT) by maintaining a policy pool and iteratively training new best responses to mixtures of previous policies~\citep{mcmahan2003planning,lanctot2017unified}. 
Such a policy-level curriculum generation paradigm is very different from the paradigm commonly used in goal-conditioned RL~\citep{matiisen2019teacher,portelas2020teacher}. Most curriculum learning methods for goal-conditioned problems directly reset the goal or initial states for each training episode to ensure the current task is of suitable difficulty for the learning agent. 
In contrast, the policy-level curriculum in zero-sum games only provides increasingly stronger opponents, and the agents are still trained by playing the full game starting from a fixed initial state distribution, which is often very challenging.

In this paper, we propose a general subgame curriculum learning framework to further accelerate MARL training for zero-sum games. It 
leverages ideas from goal-conditioned RL. 
Complementary to policy-level curriculum methods like self-play and PBT, our framework generates subgames (i.e., games induced by starting from a specific state) with growing difficulty for agents to learn and eventually solve the full game.
We provide justifications for our proposal by analyzing a simple iterated Rock-Paper-Scissors game. We show that in this game, vanilla MARL requires exponentially many samples to learn the NE. 
However, by using a buffer to store the visited states and choosing an adaptive order of state-induced subgames to learn, the NE can be learned with linear samples.

A key challenge in our framework is to choose which subgame to train on. 
This is non-trivial in zero-sum games since there does not exist a clear progression metric like the success rate in goal-conditioned problems. 
While the squared difference between the current state value and the NE value can measure the progress of learning, it is impossible to calculate this value during training as the NE is generally unknown.
We derive an alternative metric that approximates the squared difference with a bias term and a variance term.
The bias term measures how fast the state value changes and the variance term measures how uncertain the current value is.
We use the combination of the two terms as the sampling weights for states and prioritize subgames with fast change and high uncertainty. Instantiating our framework with the state selection metric and a non-parametric subgame sampler, we develop an automatic curriculum learning algorithm for zero-sum games, i.e., \emph{\underline{S}ubgame \underline{A}utomatic \underline{C}urriculum \underline{L}earning} ({\name}).
{\name} can adopt any MARL algorithm as its backbone and preserve the overall convergence property. 
In our implementation, we choose the MAPPO algorithm~\citep{yu2022mappo} for the best empirical performances.

We first evaluate {\name} in the Multi-Agent Particle Environment and Google Research Football, where {\name} learns stronger policies with lower exploitability than existing MARL algorithms for zero-sum games given the same amount of environment interactions.
We then stress-test the efficiency of {\name} in the challenging hide-and-seek environment.
{\name} leads to the emergence of all four phases of different strategies and uses 50\% fewer samples than MAPPO with self-play.

\section{Preliminary}\label{sec:prelim}
\subsection{Markov Game}
A Markov game~\citep{littman1994markov} is defined by a tuple $\mathcal{MG} = (\mathcal{N}, \mathcal{S}, \bm{\mathcal{A}}, P, \bm{R}, \gamma, \rho)$, where $\mathcal{N}=\{1, 2, \cdots, N\}$ is the set of agents, $\mathcal{S}$ is the state space,  $\bm{\mathcal{A}}=\Pi_{i=1}^N \mathcal{A}_i$ is the joint action space with $\mathcal{A}_i$ being the action space of agent $i$, $P:\mathcal{S} \times \bm{\mathcal{A}} \rightarrow \Delta(\mathcal{S})$ is the transition probability function, $\bm{R} = (R_1, R_2, \cdots, R_N):\mathcal{S} \times \bm{\mathcal{A}} \rightarrow \mathbb{R}^n$ is the joint reward function with $R_i$ being the reward function for agent $i$, $\gamma$ is the discount factor, and $\rho$ is the distribution of initial states.
Given the current state $s$ and the joint action $\bm{a}=(a_1, a_2, \cdots, a_N)$ of all agents, the game moves to the next state $s'$ with probability $P(s'|s, \bm{a})$ and agent $i$ receives a reward $R_i(s, \bm{a})$.

For infinite-horizon Markov games, a subgame $\mathcal{MG}(s)$ is defined as the Markov game induced by starting from state $s$, i.e., $\rho(s)=1$. 
Selecting subgames is therefore equivalent to setting the Markov game's initial states.
The subgames of finite-horizon Markov games are defined similarly and have an additional variable to denote the current step $t$.

We focus on two-player zero-sum Markov games, i.e., $N=2$ and $R_1(s, \bm{a}) + R_2(s, \bm{a}) = 0$ for all state-action pairs $(s, \bm{a}) \in \mathcal{S} \times \bm{\mathcal{A}}$.
We use the subscript $i$ to denote variables of player $i$ and the subscript $-i$ to denote variables of the player other than $i$.
Each player uses a policy $\pi_i:\mathcal{S} \rightarrow \mathcal{A}_i$ to produce actions and maximize its own accumulated reward.
Given the joint policy $\bm{\pi} = (\pi_1, \pi_2)$, each player's value function of state $s$ and Q-function of state-action pair $(s, \bm{a})$ are defined as 

\begin{align}
V^{\bm{\pi}}_i(s) &= \mathbb{E} \Big[ \sum_t \gamma^t R_i(s^t, \bm{a}^t) \Big| s^0=s \Big], \\
Q^{\bm{\pi}}_i(s, \bm{a}) &= \mathbb{E} \Big[ \sum_t \gamma^t R_i(s^t, \bm{a}^t) \Big| s^0=s, \bm{a}^0=\bm{a} \Big].
\end{align}

The solution concept of two-player zero-sum Markov games is Nash equilibrium (NE), a joint policy where no player can get a higher value by changing the policy alone.
\begin{definition}[NE]
A joint policy $\bm{\pi}^* = (\pi_1^*, \pi_2^*)$ is a Nash equilibrium of a Markov game if for all initial states $s^0$ with $\rho(s^0) > 0$, the following condition holds 
\begin{align}
    \pi_i^* = \argmax_{\pi_i} V_i^{(\pi_i, \pi_{-i}^*)}(s^0),\ \forall i \in \{1, 2\}.
\end{align}
\end{definition}
We use $V^*_i(\cdot)$ to denote the NE value function of player $i$ and $Q^*_i(\cdot, \cdot)$ to denote the NE Q-function of player $i$, and the following equations hold by definition and the minimax nature of zero-sum games.
\begin{align}
    V^*_i(s) &= \max_{\pi_i} \min_{\pi_{-i}} \mathbb{E}_{\bm{a} \sim \bm{\pi}(\cdot|s)} \left[Q^*_i(s, \bm{a})\right], \label{eq:v_ne} \\
    Q^*_i(s, \bm{a}) &= R_i(s, \bm{a}) + \gamma \cdot \mathbb{E}_{s' \sim P(\cdot|s, \bm{a})}\left[V^*_i(s')\right]. \label{eq:q_ne}
\end{align}

\subsection{MARL Algorithms in Zero-Sum Games}

MARL methods have been applied to zero-sum games tracing back to the TD-Gammon project~\citep{tesauro1995temporal}.
A large body of work~\citep{zinkevich2007regret,brown2019deep,steinberger2020dream,gruslys2020advantage} is based on regret minimization, and a well-known result is that the average of policies produced by self-play of regret-minimizing algorithms converges to the NE policy of zero-sum games~\citep{freund1996game}. 
Another notable line of work~\citep{littman1994markov,heinrich2015fictitious,lanctot2017unified,perolat2022mastering} combines RL algorithms with game-theoretic approaches.
These works typically use self-play or population-based training to collect samples and then apply RL methods like Q-learning~\citep{watkins1992q} and PPO~\citep{schulman2017proximal} to learn the NE value functions and policies, and have recently achieved great success~\citep{silver2016mastering,jaderberg2018human,vinyals2019grandmaster,berner2019dota}.

For the analysis in the next section, we introduce a classic MARL algorithm named minimax-Q learning~\citep{littman1994markov} that extends Q-learning to zero-sum games.
Initializing functions $Q_i(\cdot, \cdot)$ with zero values, minimax-Q uses an exploration policy induced by the current Q-functions to collect a batch of samples $\{(s^t, \bm{a}^t, r_i^t, s^{t+1})\}_{t=0}^T$ and uses these samples to update the Q-functions by
\begin{multline}
    Q_i(s^t, \bm{a}^t) \leftarrow (1 - \alpha) \cdot Q_i(s^t, \bm{a}^t) + \\
    \alpha \cdot \big(r_i^t + \gamma \cdot \max_{\pi_i} \min_{\pi_{-i}} \mathbb{E}_{\bm{a} \sim \bm{\pi}(\cdot|s)} \left[Q_i(s^{t+1}, \bm{a})\right]\big), \label{eq:minimax-q}
\end{multline}
where $\alpha$ is the learning rate.
This sample-and-update process continues until the Q-functions converge.
Under the assumptions that the state-action sets are discrete and finite and are visited an infinite number of times, it is proved that the stochastic updates by Eq.~(\ref{eq:minimax-q}) lead to the NE Q-functions~\citep{szepesvari1999unified}.

\section{A Motivating Example}\label{sec:example}

\begin{figure}[t]
    \centering
    \includegraphics[width=0.7\linewidth]{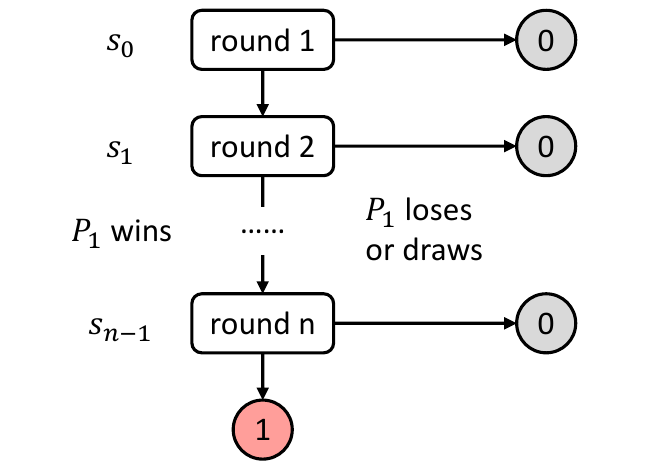}
    \vspace{-2mm}
    \caption{Illustration of the $RPS(n)$ game.}
    \label{fig:iterated_rps}
    \vspace{-4mm}
\end{figure}

\begin{figure}[t]
    \centering
    \includegraphics[width=0.7\linewidth]{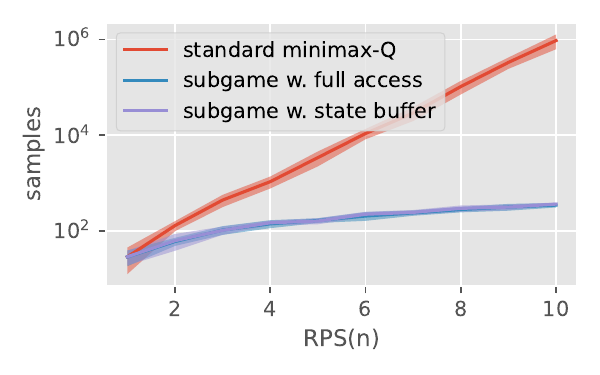}
    \vspace{-4mm}
    \caption{Number of samples used to learn the NE Q-values of $RPS(n)$ games.}
    \label{fig:iterated_rps_result}
    \vspace{-4mm}
\end{figure}

In this section, we show by a simple illustrative example that vanilla MARL methods like minimax-Q require exponentially many samples to derive the NE.
However, if we can dynamically set the initial state distribution and induce an appropriate order of subgames to learn, the sample complexity can be substantially reduced from exponential to linear. 
Such an observation motivates our proposed algorithm described in later sections.

\subsection{Iterated Rock-Paper-Scissors Game}

We introduce an iterated variant of the Rock-Paper-Scissor (RPS) game, denoted as $RPS(n)$.
As shown in Fig.~\ref{fig:iterated_rps}, $P_1$ and $P_2$ play the RPS game for up to $n$ rounds.
If $P_1$ wins all rounds, it gets a reward of $1$ and $P_2$ gets a reward of $-1$.
If $P_1$ loses or draws in any round, the game ends immediately without playing the remaining rounds and both players get zero rewards.
Note that the $RPS(n)$ game is different from playing the RPS game repeatedly for $n$ times because players can play less than $n$ rounds and they only receive a non-zero reward if $P_1$ wins in all rounds.
We use $s_k$ to denote the state where players have already played $k$ RPS games and are at the $k + 1$ round. 
It is easy to verify that the NE policy for both players is to play Rock, Paper, or Scissors with equal probability at each state.
Under this joint NE policy, $P_1$ can win one RPS game with $1/3$ probability, and the probability for $P_1$ to win all $n$ rounds and get a non-zero reward is $1/3^n$.

Consider using standard minimax-Q learning to solve the $RPS(n)$ game.
With Q-functions initialized to zero, we execute the exploration policy to collect samples and perform the update in Eq.~(\ref{eq:minimax-q}).
Note all state-actions pairs are required to be visited to guarantee convergence to the NE. Therefore, in this sparse-reward game, random exploration will clearly take ${\mathcal{O}}(3^n)$ steps to get a non-zero reward. 
Moreover, even if the exploration policy is perfectly set to the NE policy, the probability for $P_1$ to get the non-zero reward by winning all RPS games is still $\mathcal{O}(1/3^n)$, 
requiring at least $\mathcal{O}(3^n)$ samples to learn the NE Q-values of the $RPS(n)$ game.

\begin{algorithm}[t]
\caption{Subgame curriculum learning}\label{alg:framework}
\KwIn{state sampler $\mathrm{oracle}(\cdot)$.}
Initialize policy $\bm{\pi}$\;
\Repeat{$\bm{\pi}$ converges}
{
    Sample $s^0$ $\sim \mathrm{oracle}(\mathcal{\mathcal{S}})$\;
    Rollout $\bm{\pi}$ in $\mathcal{MG}(s^0)$\;
    Train $\bm{\pi}$ via MARL\;
}
\KwOut{final policy $\bm{\pi}$.}
\end{algorithm}

\subsection{From Exponential to Linear Complexity}
An important observation is that the states in later rounds become exponentially rare in the samples generated by starting from the fixed initial state.
If we can directly reset the game to these states and design a smart order of minimax-Q updates on the subgames induced by these states, the NE learning can be accelerated significantly.
Note that $RPS(n)$ can be regarded as the composition of $n$ individual $RPS(1)$ games, a suitable order of learning would be from the easiest subgame $RPS(1)$ starting from state $s_{n-1}$ to the full game $RPS(n)$ starting from state $s_0$.
Assuming we have full access to the state space, we first reset the game to $s_{n-1}$ and use minimax-Q to solve subgame $RPS(1)$ with $\mathcal{O}(1)$ samples.
Given that the NE Q-values of $RPS(k)$ are learned, the next subgame $RPS(k+1)$ is equivalent to an $RPS(1)$ game where the winning reward is the value of state $s_{n-k}$.
By sequentially applying minimax-Q to solve all $n$ subgames from $RPS(1)$ to $RPS(n)$, the number of samples required to learn the NE Q-values is reduced substantially from $\mathcal{O}(3^n)$ to $\mathcal{O}(n)$.

In practice, we usually do not have access to the entire state space and cannot directly start from the last subgame $RPS(1)$.
Instead, we can use a buffer to store all visited states and gradually span the state space.
By resetting games to the newly visited states, the number of samples required to cover the full state space is still $\mathcal{O}(n)$, and we can then apply minimax-Q from $RPS(1)$ to $RPS(n)$.
Therefore, the total number of samples is still $\mathcal{O}(n)$.
We validate our analysis by running experiments on $RPS(n)$ games for $n=1,\cdots,10$ and the results averaged over ten seeds are shown in Fig.~\ref{fig:iterated_rps_result}.
It can be seen that the sample complexity reduces from exponential to linear by running minimax-Q over a smart order of subgames, and the result of using a state buffer in practice is comparable to the result with full access.
The detailed analysis can be found in Appendix~\ref{sec:app:example}.
\footnote{Appendix can be found at \url{http://arxiv.org/abs/2310.04796}.}

\section{Method}\label{sec:method}
\begin{algorithm}[t]
\begin{minipage}{0.9\linewidth}
\caption{Subgame Automatic Curriculum Learning ({\name})}
\label{alg:full}
\KwIn{state buffers $\mathcal{M}$ with capacity $K$, probability $p$ to sample initial state from the state buffer.}
Randomly initialize policy $\pi_i$ and value function $V_i$ for player $i=1, 2$\;
\Repeat{$(\pi_1, \pi_2)$ converges}
{
    $V'_i \gets V_i,\ i=1, 2$\;
    \tcp{Select subgame and train policy.}
    \For{each parallel environment}
    {
        Sample $s^0 \sim \mathrm{sampler}(\mathcal{M})$ with probability $p$, else $s^0 \sim \rho(\cdot)$\;
        Rollout in $\mathcal{MG}(s^0)$ and collect samples\;
    }
    Train $\{\pi_i, V_i\}_{i=1}^{2}$ via MARL\;
    \tcp{Compute weight by Eq.~(\ref{eq:estimate}) and update state buffer.}
    $\tilde{w}^t \gets \alpha \cdot \mathbb{E}[\tilde{V}_i(s^t) - \tilde{V}'_i(s^t)]^2
    + \Var(\{\tilde{V}_i(s^t)\}_{i=1}^2),\\
    \ t=0, \cdots, T$\;    
    $\mathcal{M} \gets \mathcal{M} \cup \{(s^t, \tilde{w}^t)\}_{t=0}^T$\;
    \If{$\|\mathcal{M}\| > K$}
    {
        $\mathcal{M} \gets \mathrm{FPS}(\mathcal{M}, K)$\;
    }
}
\KwOut{final policy $(\pi_1, \pi_2)$.}
\end{minipage}
\end{algorithm}

The motivating example suggests that NE learning can be largely accelerated by running MARL algorithms in a smart order over states.
Inspired by this insight, we present a general framework to accelerate NE learning in zero-sum games by training over a curriculum of subgames.
We further propose two practical techniques to instantiate the framework and present the overall algorithm.

\subsection{Subgame Curriculum Learning}

The key issue of the standard sample-and-update framework is that the rollout trajectories always start from the fixed initial state distribution $\rho$, so visiting states that are most critical for efficient learning can consume a large number of samples. 
To accelerate training, we can directly reset the environment to those critical states.
Suppose we have an oracle state sampler $\textrm{oracle}(\cdot)$ that can initiate suitable states for the current policy to learn, i.e., generate appropriate induced subgames, we can derive a general-purpose framework in Alg.~\ref{alg:framework}, which we call subgame curriculum learning. Note that this framework is compatible with any MARL algorithm for zero-sum Markov games.\\
A desirable feature of subgame curriculum learning is that it does not change the convergence property of the backbone MARL algorithm, as discussed below.

\begin{restatable}{proposition}{correctness}\label{prop:correctness}
If all initial states $s^0$ with $\rho(s^0) > 0$ are sampled infinitely often, and the backbone MARL algorithm is guaranteed to converge to an NE in zero-sum Markov games, then subgame curriculum learning also produces an NE of the original Markov game.
\end{restatable}
The proof can be found in Appendix~\ref{sec:app:proof}.
Note that such a requirement is easy to satisfy. For example, given any state sampler $\textrm{oracle}(\cdot)$, we can construct a valid mixed sampler by sampling from $\textrm{oracle}(\cdot)$ for probability $0<p<1$ and sampling from $\rho$ for probability $1-p$.

\textbf{Remark.}
With a given state sampler, the only requirement of our subgame curriculum learning framework is that the environment can be reset to a desired state to generate the induced game.
This is a standard assumption in the curriculum learning literature~\citep{florensa2018automatic,matiisen2019teacher,portelas2020teacher} and is feasible in many RL environments. 
For environments that do not support this feature, we can simply reimplement the reset function to make them compatible with our framework.

\begin{figure*}[t]
    \centering
    \includegraphics[width=0.8\textwidth]{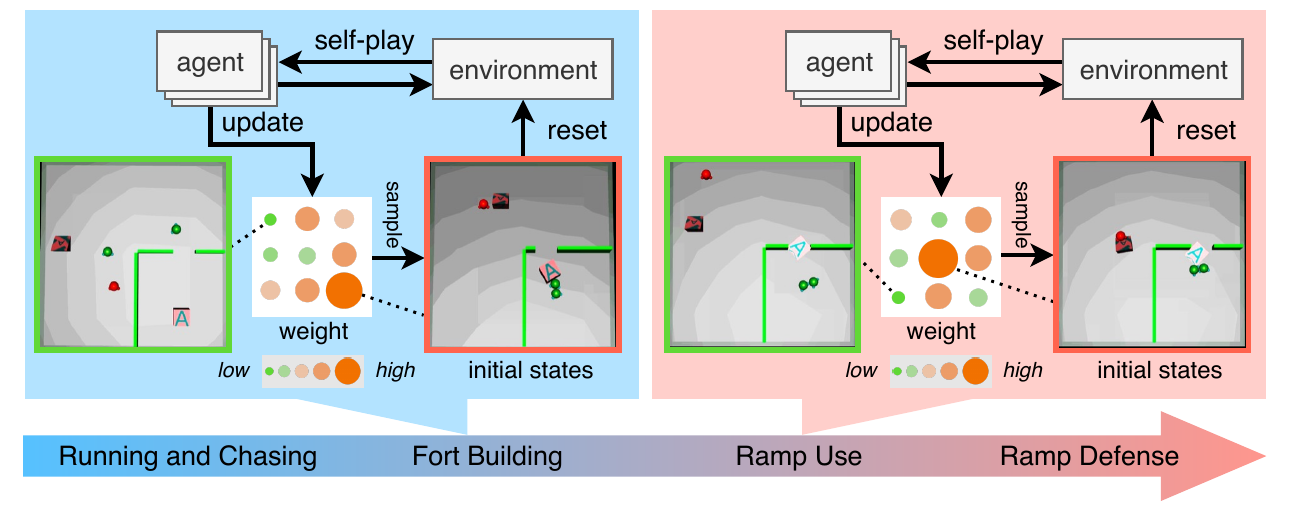}
    \vspace{-4mm}
    \caption{Illustration of {\name} in the hide-and-seek environment. In the Fort Building stage, the states with hiders near the box have high weights (red) and agents can easily learn to build a fort by practicing on these subgames, while the states with randomly spawned hiders have low weights (green) and contribute less to learning.}
    \label{fig:overview}
    \vspace{-4mm}
\end{figure*}

\subsection{Subgame Sampling Metric}

A key question is how to instantiate the oracle sampler, i.e., \emph{which subgame should we train on for faster convergence}?
Intuitively, for a particular state $s$, if its value has converged to the NE value, that is, $V_i(s)=V_i^*(s)$, we should no longer train on the subgame induced by it.
By contrast, if the gap between its current value and the NE value is substantial, we should probably train more on the induced subgame. 
Thus, a simple way is to use the squared difference of the current value and the NE value as the weight for a state and sample states with probabilities proportional to the weights.
Concretely, the state weight can be written as
\begin{align}
    w(s) &= \frac{1}{2} \sum_{i=1}^2(V_i^*(s) - V_i(s))^2 \\
    &= \mathbb{E}_i\big[(V_1^*(s) - \tilde{V}_i(s))^2\big] \\
    &= \mathbb{E}_i\big[V_1^*(s) - \tilde{V}_i(s)\big]^2 + \mathrm{Var}_i\big[V_1^*(s)-\tilde{V}_i(s)\big], \label{eq:oracle}
\end{align}
where $\tilde{V}_1(s) = V_1(s)$ and $\tilde{V}_2(s) = -V_2(s)$.
The second equality holds because the game is zero-sum and $V_2^*(s) = -V_1^*(s)$.
With random initialization and different training samples, $\{\tilde{V}_i\}_{i=1}^2$ can be regarded as an ensemble of two value functions, and the weight $w(s)$ becomes the expectation over the ensemble.
The last equality further expands the expectation to a bias term and a variance term, and we sample state with probability $P(s) = w(s) / \sum_{s'}w(s')$.
For the motivating example of $RPS(n)$ game, the NE value decreases exponentially from the last state $s_{n-1}$ to the initial state $s_0$.
With value functions initialized close to zero, the prioritized subgames throughout training will move gradually from the last round to the first round, which is approximately the optimal order.

However, Eq.~(\ref{eq:oracle}) is very hard to compute in practice because the NE value is generally unknown.
Inspired by Eq.~(\ref{eq:oracle}), we propose the following alternative state weight
\begin{small}
{
\begin{align}
    \tilde{w}(s) =  \alpha \cdot \mathbb{E}_i\big[\tilde{V}_i^{{\pi}_i^{(t)}}(s) - \tilde{V}^{{\pi}_i^{(t-1)}}_i(s)\big]^2 + \mathrm{Var}_i\big[\tilde{V}_i(s)\big], \label{eq:estimate}
\end{align}
}
\end{small}
which takes a hyperparameter $\alpha$ and uses the difference between two consecutive value function checkpoints instead of the difference between the NE value and the current value in Eq.~(\ref{eq:oracle}).
The first term in Eq.~(\ref{eq:estimate}) measures how fast the value functions change over time.
If this term is large, the value functions are changing constantly and still far from the NE value; if this term is marginal, the value functions are probably close to the converged NE value.
The second term in Eq.~(\ref{eq:estimate}) measures the uncertainty of the current learned values and is the same as the variance term in Eq.~(\ref{eq:oracle}) because $V^*_1(s)$ is a constant. 
If $\alpha=1$, Eq.~(\ref{eq:estimate}) approximates Eq.~(\ref{eq:oracle}) as $t$ increases.
It is also possible to train an ensemble of value functions for each player to further improve the empirical performance.
Additional analysis can be found in Appendix~\ref{sec:app:metric}.

Since Eq.~(\ref{eq:estimate}) does not require the unknown NE value to compute, it can be used in practice as the weight for state sampling and can be implemented for most MARL algorithms. 
By selecting states with fast value change and high uncertainty, our framework prioritizes subgames where agents' performance can quickly improve through learning.

\subsection{Particle-based Subgame Sampler}\label{sec:method:sampler}

With the sample weight at hand, we can generate subgames by sampling initial states from the state space.
But it is impractical to sample from the entire space which is usually unavailable and can be exponentially large for complex games.
Typical solutions include training a generative adversarial network (GAN)~\citep{dendorfer2020goal} or using a parametric Gaussian mixture model (GMM)~\citep{portelas2020teacher} to generate states for automatic curriculum learning.
However, parametric models require a large number of samples to fit accurately and cannot adapt instantly to the ever-changing weight in our case.
Moreover, the distribution of weights is highly multi-modal, which is hard to capture for many generative models.

We instead adopt a particle-based approach and maintain a large state buffer $\mathcal{M}$ using all visited states throughout training to approximate the state space.
Since the size of the buffer is limited while the state space can be infinitely large, it is important to keep representative samples that are sufficiently far from each other to ensure good coverage of the state space.
When the number of states exceeds the buffer's capacity $K$, we use farthest point sampling (FPS)~\citep{qi2017pointnet} which iteratively selects the farthest point from the current set of points.
In our implementation, we first normalize each dimension of the states and the distance between two states is simply the Euclidean distance. More details can be found in Appendix~\ref{sec:app:fps}.

\begin{figure*}[t]
\centering
\subfigure[MPE: exploitability.]
{
    \includegraphics[width=0.285\textwidth]{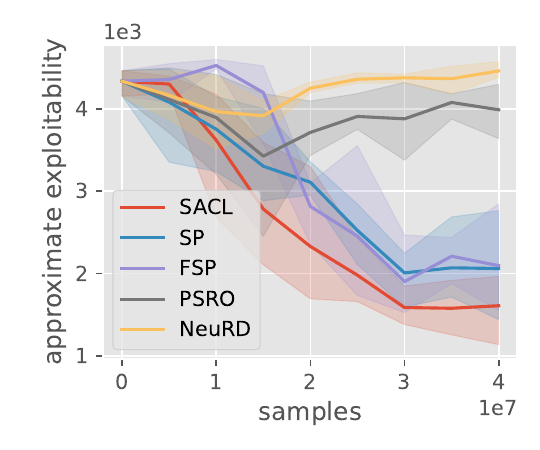}
    \label{fig:exp:mpe_default}
}
\hfill
\subfigure[MPE hard: exploitability.]
{
    \includegraphics[width=0.285\textwidth]{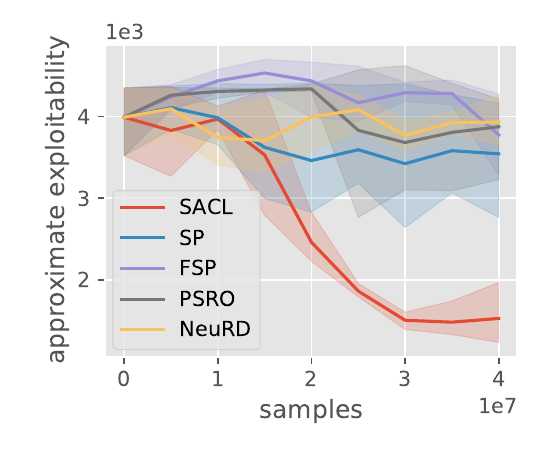}
    \label{fig:exp:mpe_hard}
}
\hfill
\subfigure[HnS: number of samples.]
{
    \includegraphics[width=0.38\textwidth]{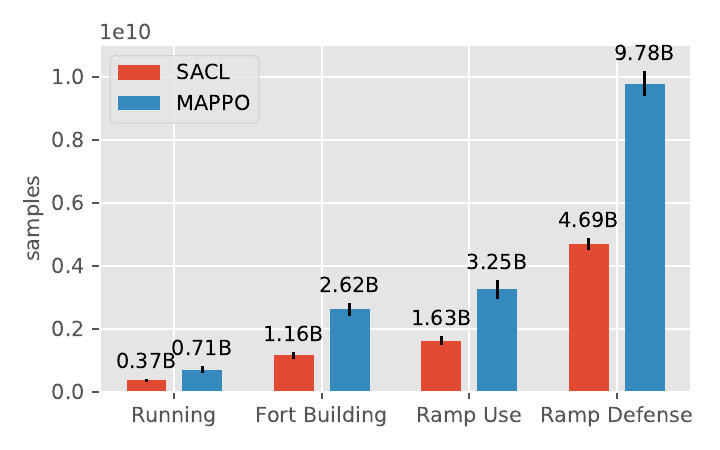}
    \label{fig:exp:hns_samples}
}
\vspace{-2mm}
\caption{Main experiment results in (a) MPE, (b) MPE hard, and (c) Hide-and-seek.}
\label{fig:exp:mpe_and_hns}
\end{figure*}

\begin{table*}[t]
\begin{center}
\begin{tabularx}{0.9\textwidth}{@{\extracolsep{\fill}}*{6}{c}}
    \toprule
    Scenario & {\name} & SP & FSP & PSRO & NeuRD \\
    \midrule
    pass and shoot & \textbf{0.35 (0.13)} & 0.48 (0.31) & 0.83 (0.10) & 0.80 (0.09) & 0.79 (0.15) \\
    run pass and shoot & \textbf{0.60 (0.04)} & 0.68 (0.09) & 0.78 (0.08) & 0.83 (0.04) & 0.95 (0.04) \\
    3 vs 1 with keeper & \textbf{0.45 (0.06)} & 0.83 (0.03) & 0.63 (0.25) & 0.85 (0.05) & 0.81 (0.16) \\
    \bottomrule
\end{tabularx}
\end{center}
\vspace{-2mm}
\caption{Approximate exploitability of learned policies in different GRF scenarios.}
\vspace{-5mm}
\label{tab:exp:grf}
\end{table*}

\subsection{Overall Algorithm}
Combining the subgame sampling metric and the particle-based sampler, we present a realization of the subgame curriculum learning framework, i.e., the \emph{\underline{S}ubgame \underline{A}utomatic \underline{C}urriculum \underline{L}earning} ({\name}) algorithm, which is summarized in Alg.~\ref{alg:full}. When each episode resets, we use the particle-based sampler to generate suitable initial states $s_0$ for the current policy to learn.
To satisfy the requirements in Proposition~\ref{prop:correctness}, we also reset the game according to the initial state distribution $\rho(\cdot)$ with $0.3$ probability.
After collecting a number of samples, we train the policies and value functions using MARL.
The weights for the newly collected states are computed according to Eq.~(\ref{eq:estimate}) and used to update the state buffer $\mathcal{M}$.
If the capacity of the state buffer is exceeded, we use FPS to select representative states-weight pairs and delete the others.
An overview of {\name} in the hide-and-seek game is illustrated in Fig.~\ref{fig:overview}.

\section{Experiment}\label{sec:expr}
\begin{figure*}[t]
\centering
\subfigure[Fort Building.]
{
    \includegraphics[width=0.19\textwidth]{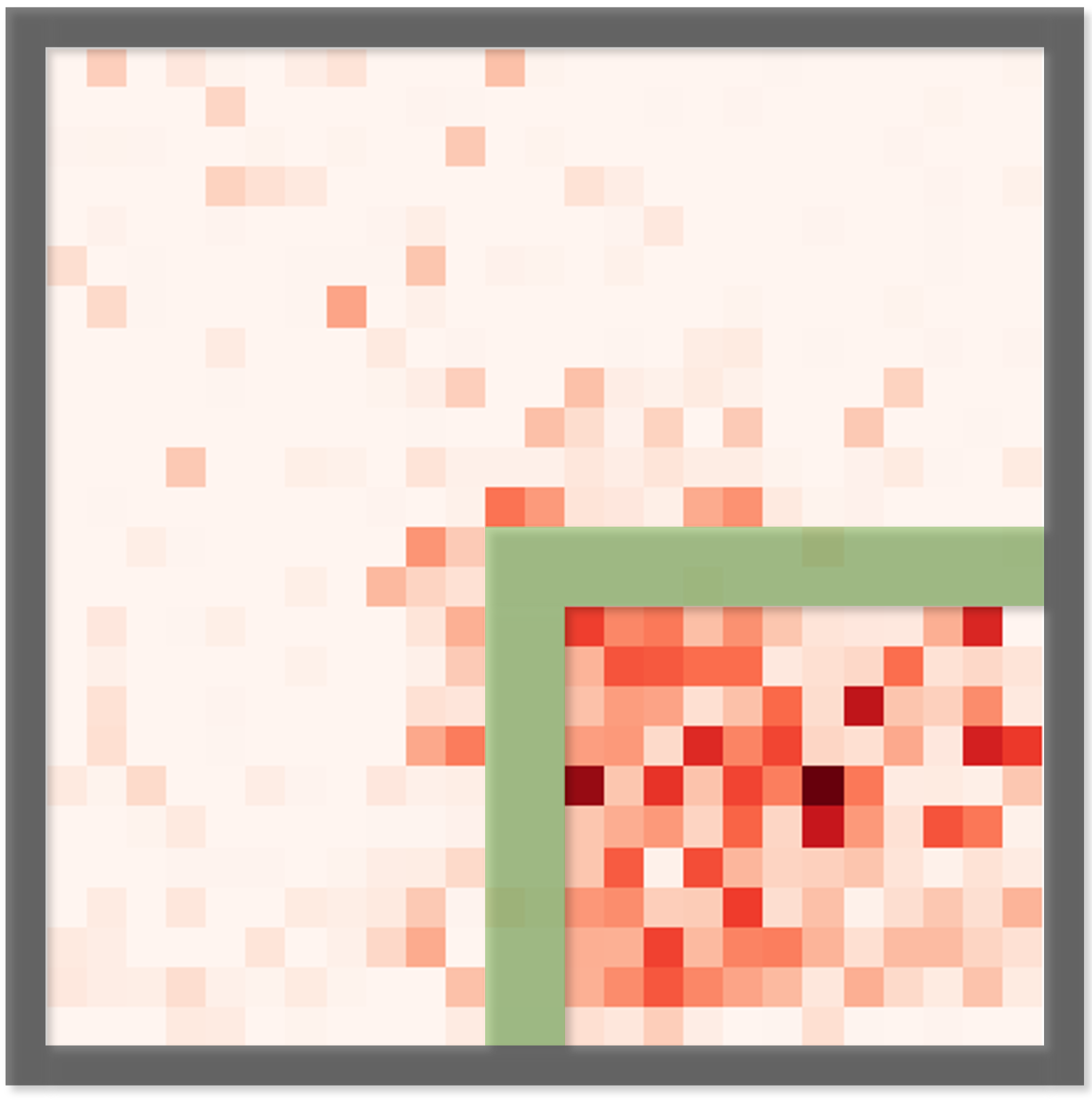}
    \label{fig:exp:hns_fort}
}
\hfill
\subfigure[Ramp Use.]
{
    \includegraphics[width=0.19\textwidth]{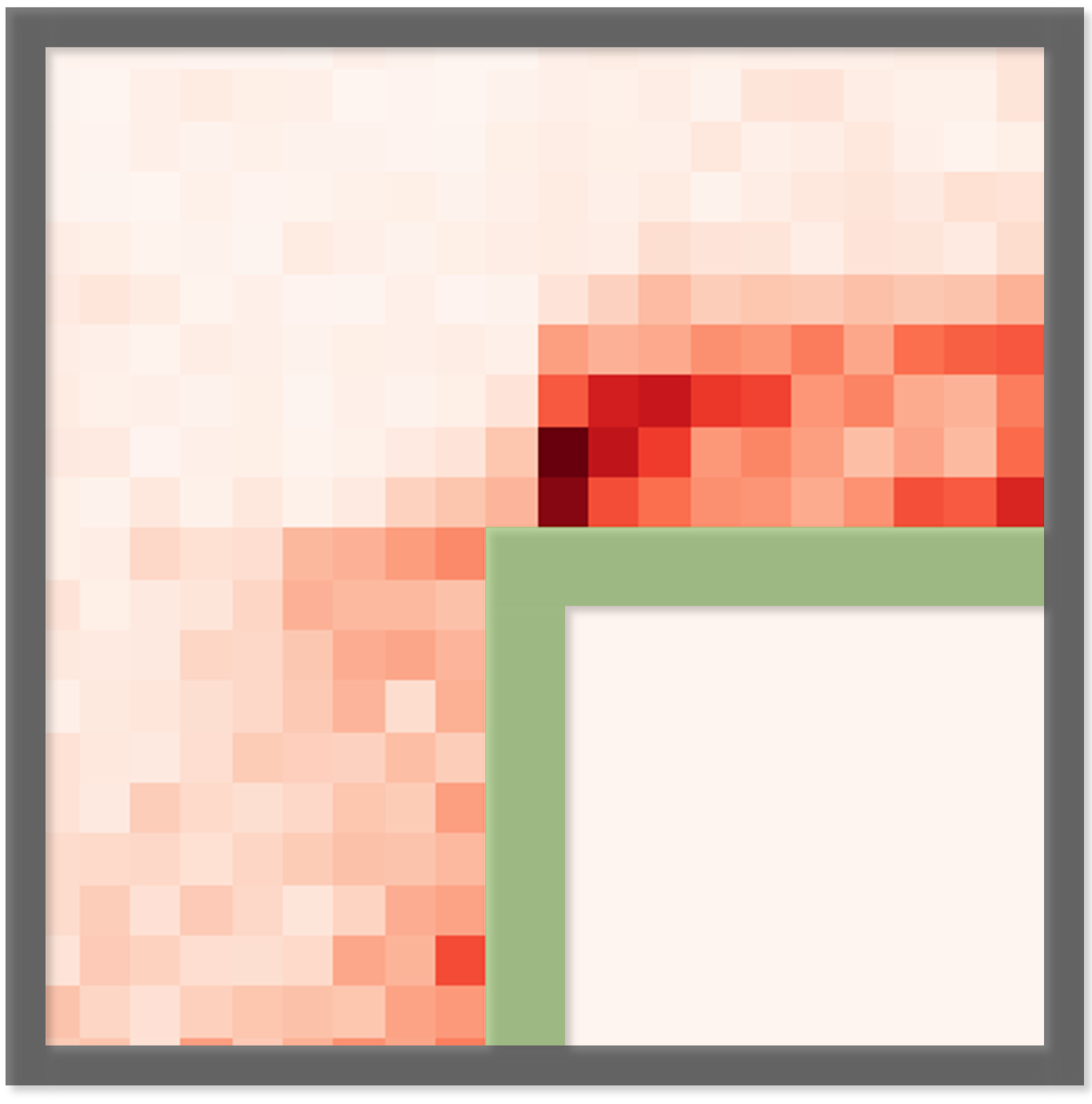}
    \label{fig:exp:hns_ramp}
}
\hfill
\subfigure[Ablation on metric.]
{
    \includegraphics[width=0.25\textwidth]{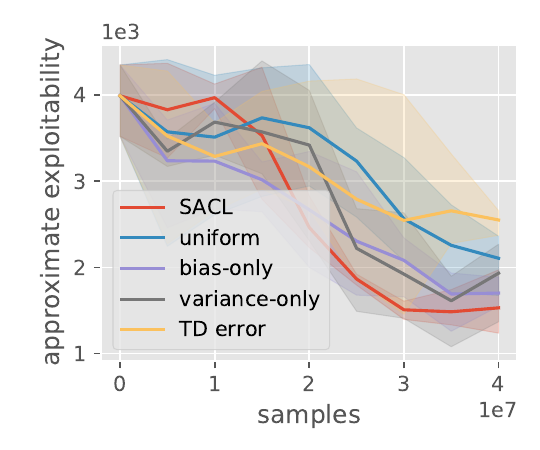}
    \label{fig:exp:ablation_metric}
}
\hfill
\subfigure[Ablation on generator.]
{
    \includegraphics[width=0.25\textwidth]{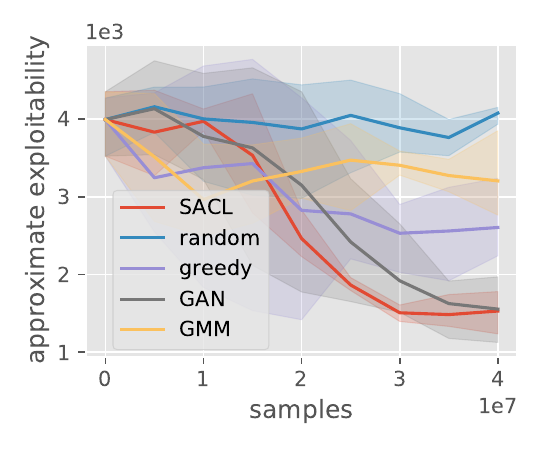}
    \label{fig:exp:ablation_sampler}
}
\caption{Visualization of the state distributions in HnS (a-b) and ablation studies (c-d).}
\label{fig:exp:hns}
\end{figure*}

We evaluate {\name} in three different zero-sum environments: Multi-Agent Particle Environment (MPE)~\citep{lowe2017multi}, Google Research Football (GRF)~\citep{kurach2020google}, and the hide-and-seek (HnS) environment~\citep{baker2020emergent}.
We use a state-of-the-art MARL algorithm MAPPO~\citep{yu2022mappo} as the backbone in all experiments. We evaluate the performance of policies by exploitability. How to define and compute the exploitability can be found in Appendix~\ref{sec:app:exploitability}.

\subsection{Main Results}
We first compare the performance of {\name} in three environments against the following baselines for solving zero-sum games: self-play (SP), two popular variants including Fictitious Self-Play (FSP)~\citep{heinrich2015fictitious} and Neural replicator dynamics (NeuRD)~\citep{hennes2020neural}, and a population-based training method policy-space response oracles (PSRO)~\citep{lanctot2017unified}.
More implementation details can be found in Appendix~\ref{sec:app:training}.

\textbf{Multi-Agent Particle Environment.}
We consider the \textit{predator-prey} scenario in MPE, where three slower cooperating predators chase one faster prey in a square space with two obstacles. 
In the default setting, all agents are spawned uniformly in the square. 
We also consider a harder setting where the predators are spawned in the top-right corner and the prey is spawned in the bottom-left corner.
All algorithms are trained for 40M environment samples and the curves of approximate exploitability w.r.t. sample over three seeds are shown in Fig.~\ref{fig:exp:mpe_default} and \ref{fig:exp:mpe_hard}.
{\name} converges faster and achieves lower exploitability than all baselines in both settings, and its advantage is more obvious in the hard scenario.
This is because the initial state distribution in corners makes the full game challenging to solve, while {\name} generates an adaptive state distribution and learns on increasingly harder subgames to accelerate NE learning.
More results and discussions can be found in Appendix~\ref{sec:app:mpe}.

\textbf{Google Research Football.}
We evaluate {\name} in three GRF academy scenarios, namely \textit{pass and shoot}, \textit{run pass and shoot}, and \textit{3 vs 1 with keeper}.
In all scenarios, the left team's agents cooperate to score a goal and the right team's agents try to defend them.
The first scenario is trained for 300M environment samples and the last two scenarios are trained for 400M samples.
Table~\ref{tab:exp:grf} lists the approximate exploitabilities of different methods' policies over three seeds, and {\name} achieves the lowest exploitability.
Additional cross-play results and discussions can be found in Appendix~\ref{sec:app:grf}.

\textbf{Hide-and-seek environment.}
HnS is a challenging zero-sum game with known NE policies, which makes it possible for us to directly evaluate the number of samples used for NE convergence.
We consider the \textit{quadrant} scenario where there is a room with a door in the lower right corner.
Two hiders, one box, and one ramp are spawned uniformly in the environment, and one seeker is spawned uniformly outside the room.
Both the box and the ramp can be moved and locked by agents.
The hiders aim to avoid the lines of sight from the seeker while the seeker aims to find the hiders.

There is a total of four stages of emergent stages in HnS, i.e., Running and Chasing, Fort Building, Ramp Use, and Ramp Defense.
As shown in Fig.~\ref{fig:exp:hns_samples}, {\name} with MAPPO backbone produces all four stages and converges to the NE policy with only 50\% the samples of MAPPO with self-play.
We also visualize the initial state distribution to show how {\name} selects appropriate subgames for agents to learn.
Fig.~\ref{fig:exp:hns_fort} depicts the distribution of hiders' position in the Fort Building stage.
The probabilities of states with hiders inside the room are much higher than states with hiders outside, making it easier for hiders to learn to build a fort with the box. 
Similarly, the distribution of the seeker's position in the Ramp Use stage is shown in Fig.~\ref{fig:exp:hns_ramp}, and the most sampled subgames start from states where the seeker is close to the walls and is likely to use the ramp. 

\subsection{Ablation Study}
We perform ablation studies to examine the effectiveness of the proposed sampling metric and particle-based sampler.
All experiments are done in the hard \textit{predator-prey} scenario of MPE and the results are averaged over three seeds.
More ablation studies on state buffer size, subgame sample probability, and other hyperparameters can be found in Appendix~\ref{sec:app:mpe}.

\textbf{Subgame sampling metric.}
The sampling metric used in {\name} follows Eq.~(\ref{eq:estimate}) which consists of a bias term and a variance term. 
We compare it with five other metrics including a uniform metric, a bias-only metric, a variance-only metric and a temporal difference (TD) error metric. The last metric uses the TD error $|\delta_{t}| = |r^t + \gamma V(s^{t+1}) - V(s^t)|$ as the weight, which can be regarded as an estimation of value uncertainty.
The results are shown in Fig.~\ref{fig:exp:ablation_metric} and the sampling metric used by {\name} outperforms both the bias-only metric and variance-only metric. 

\textbf{State generator.}
We substitute the particle-based sampler with other state generators including using GAN from the work~\citep{dendorfer2020goal} and using GMM from the work~\citep{portelas2020teacher}. 
We also replace the FPS buffer update method with a uniform one that randomly keeps states and a greedy one that keeps states with the highest weights.
Results in Fig.~\ref{fig:exp:ablation_metric} show that our particle-based sampler with FPS update leads to the fastest convergence and lowest exploitability.

\section{Related Work}\label{sec:related}
A large number of works achieve faster convergence in zero-sum games by playing against an increasingly stronger policy.
The most popular methods are self-play and its variants~\citep{heinrich2016deep,bai2020near,jin2021v,perolat2022mastering}.
Self-play creates a natural curriculum and leads to emergent complex skills and behaviors~\citep{bansal2018emergent,baker2020emergent}.
Population-based training like double oracle~\citep{mcmahan2003planning} and policy-space response oracles (PSRO)~\citep{lanctot2017unified} extend self-play by training a pool of policies.
Some follow-up works further accelerate training by constructing a smart mixing strategy over the policy pool according to the policy landscape~\citep{balduzzi2019open,perez2021modelling,liu2021towards,feng2021neural}.
\citep{mcaleer2021xdo} extends PSRO to extensive-form games by building policy mixtures at all states rather than only the initial states, but it still directly solves the full game starting from some fixed states.

In addition to policy-level curriculum learning methods, other works to accelerate training in zero-sum games usually adopt heuristics and domain knowledge like the number of agents~\citep{long2020evolutionary,wang2020few} or environment specifications~\citep{berner2019dota,serrino2019finding,tang2021discovering}.
By contrast, our method automatically generates a curriculum over subgames without domain knowledge and only requires the environments can be reset to desired states.
Subgame-solving technique~\citep{brown2017safe} is also used in online strategy refinement to improve the blueprint strategy of a simplified abstract game.
Another closely related work to our method is \citep{chen2021temporal} which combines backward induction with policy learning, but this method requires knowledge of the game topology and can only be applied to finite-horizon Markov games.

Besides zero-sum games, curriculum learning is also studied in cooperative settings.
The problem is often formalized as goal-conditioned RL where the agents need to reach a specific goal in each episode. 
Curriculum learning methods design or train a smart sampler to generate proper task configurations or goals that are most suitable for training advances w.r.t. some progression metric~\citep{chen2016variational,florensa2017reverse,florensa2018automatic,racaniere2019automated,matiisen2019teacher,portelas2020teacher,dendorfer2020goal}.
Such a metric typically relies on an explicit signal, such as the goal-reaching reward, success rates, or the expected value of the testing tasks. 
However, in the setting of zero-sum games, these explicit progression metrics become no longer valid since the value associated with a Nash equilibrium can be arbitrary. 
A possible implicit metric is value disagreement~\citep{zhang2020automatic} used in goal-reaching tasks, which can be regarded as the variance term in our metric.
By adding a bias term, our metric approximates the squared distance to NE values and gives better results in ablation studies.

Our work adopts a non-parametric subgame sampler which is fast to learn and naturally multi-modal, instead of training an expensive deep generative model like GAN~\citep{florensa2018automatic}. 
Such an idea has been recently popularized in the literature. 
Some representative samplers are Gaussian mixture model~\citep{warde2018unsupervised}, Stein variational inference~\citep{chen2021variational}, Gaussian process~\citep{mehta2020active}, or simply evolutionary computation~\citep{wang2019poet,wang2020enhanced}. 
Technically, our method is related to prioritized experience replay~\citep{schaul2015prioritized,florensa2017reverse,li2022pmr} with the difference that we maintain a buffer~\citep{warde2018unsupervised} to approximate the uniform distribution over the state space. 
Our method is also related to episodic memory replay~\citep{blundell2016model,pritzel2017neural} which stores past experience and chooses actions based on previous success when similar states are encountered.
By contrast, our method proactively resets the environment to intermediate states and collects experience in the sampled subgame.

\section{Conclusion}\label{sec:conclusion}
We present {\name}, a general algorithm for accelerating MARL training in zero-sum Markov games based on the subgame curriculum learning framework.
We propose to use the approximate squared distance to NE values as the sampling metric and use a particle-based sampler for subgame generation.
Instead of starting from the fixed initial states, RL agents trained with {\name} can practice more on subgames that are most suitable for the current policy to learn, thus boosting training efficiency.
We report appealing experiment results that {\name} efficiently discovers all emergent strategies in the challenging hide-and-seek environment and uses only half the samples of MAPPO with self-play.
We hope {\name} can be helpful to speed up prototype development and help make MARL training on complex zero-sum games more affordable to the community.

\clearpage
\section{Acknowledgments}
This research was supported by the National Natural Science Foundation of China (No.62325405, U19B2019, M-0248), Tsinghua University Initiative Scientific Research Program, Tsinghua-Meituan Joint Institute for Digital Life, Beijing National Research Center for Information Science, Technology (BNRist) and Beijing Innovation Center for Future Chips.

\bibliography{aaai24}

\begin{thebibliography}{62}
\providecommand{\natexlab}[1]{#1}

\bibitem[{Bai, Jin, and Yu(2020)}]{bai2020near}
Bai, Y.; Jin, C.; and Yu, T. 2020.
\newblock Near-optimal reinforcement learning with self-play.
\newblock \emph{Advances in neural information processing systems}, 33: 2159--2170.

\bibitem[{Baker et~al.(2020)Baker, Kanitscheider, Markov, Wu, Powell, McGrew, and Mordatch}]{baker2020emergent}
Baker, B.; Kanitscheider, I.; Markov, T.; Wu, Y.; Powell, G.; McGrew, B.; and Mordatch, I. 2020.
\newblock Emergent Tool Use From Multi-Agent Autocurricula.
\newblock In \emph{International Conference on Learning Representations}.

\bibitem[{Balduzzi et~al.(2019)Balduzzi, Garnelo, Bachrach, Czarnecki, Perolat, Jaderberg, and Graepel}]{balduzzi2019open}
Balduzzi, D.; Garnelo, M.; Bachrach, Y.; Czarnecki, W.; Perolat, J.; Jaderberg, M.; and Graepel, T. 2019.
\newblock Open-ended learning in symmetric zero-sum games.
\newblock In \emph{International Conference on Machine Learning}, 434--443. PMLR.

\bibitem[{Bansal et~al.(2018)Bansal, Pachocki, Sidor, Sutskever, and Mordatch}]{bansal2018emergent}
Bansal, T.; Pachocki, J.; Sidor, S.; Sutskever, I.; and Mordatch, I. 2018.
\newblock Emergent Complexity via Multi-Agent Competition.
\newblock In \emph{International Conference on Learning Representations}.

\bibitem[{Berner et~al.(2019)}]{berner2019dota}
Berner, C.; et~al. 2019.
\newblock Dota 2 with large scale deep reinforcement learning.
\newblock \emph{arXiv preprint arXiv:1912.06680}.

\bibitem[{Blundell et~al.(2016)}]{blundell2016model}
Blundell, C.; et~al. 2016.
\newblock Model-free episodic control.
\newblock \emph{arXiv preprint arXiv:1606.04460}.

\bibitem[{Brown et~al.(2019)Brown, Lerer, Gross, and Sandholm}]{brown2019deep}
Brown, N.; Lerer, A.; Gross, S.; and Sandholm, T. 2019.
\newblock Deep counterfactual regret minimization.
\newblock In \emph{International conference on machine learning}, 793--802. PMLR.

\bibitem[{Brown and Sandholm(2017)}]{brown2017safe}
Brown, N.; and Sandholm, T. 2017.
\newblock Safe and nested subgame solving for imperfect-information games.
\newblock \emph{Advances in neural information processing systems}, 30.

\bibitem[{Burch et~al.(2014)}]{burch2014solving}
Burch, N.; et~al. 2014.
\newblock Solving imperfect information games using decomposition.
\newblock In \emph{Proceedings of the AAAI Conference on Artificial Intelligence}, volume~28.

\bibitem[{Chen et~al.(2021{\natexlab{a}})Chen, Zhang, Xu, Ma, Yang, Song, Wang, and Wu}]{chen2021variational}
Chen, J.; Zhang, Y.; Xu, Y.; Ma, H.; Yang, H.; Song, J.; Wang, Y.; and Wu, Y. 2021{\natexlab{a}}.
\newblock Variational Automatic Curriculum Learning for Sparse-Reward Cooperative Multi-Agent Problems.
\newblock \emph{Advances in Neural Information Processing Systems}, 34: 9681--9693.

\bibitem[{Chen et~al.(2021{\natexlab{b}})Chen, Zhou, Wu, and Fang}]{chen2021temporal}
Chen, W.; Zhou, Z.; Wu, Y.; and Fang, F. 2021{\natexlab{b}}.
\newblock Temporal Induced Self-Play for Stochastic Bayesian Games.
\newblock \emph{arXiv preprint arXiv:2108.09444}.

\bibitem[{Chen et~al.(2016)}]{chen2016variational}
Chen, X.; et~al. 2016.
\newblock Variational Lossy Autoencoder.
\newblock \emph{arXiv preprint arXiv:1611.02731}.

\bibitem[{Cobbe et~al.(2021)}]{cobbe2021phasic}
Cobbe, K.~W.; et~al. 2021.
\newblock Phasic policy gradient.
\newblock In \emph{International Conference on Machine Learning}, 2020--2027. PMLR.

\bibitem[{Dendorfer, Osep, and Leal-Taix{\'e}(2020)}]{dendorfer2020goal}
Dendorfer, P.; Osep, A.; and Leal-Taix{\'e}, L. 2020.
\newblock Goal-gan: Multimodal trajectory prediction based on goal position estimation.
\newblock In \emph{Proceedings of the Asian Conference on Computer Vision}.

\bibitem[{Feng et~al.(2021)Feng, Slumbers, Wan, Liu, McAleer, Wen, Wang, and Yang}]{feng2021neural}
Feng, X.; Slumbers, O.; Wan, Z.; Liu, B.; McAleer, S.; Wen, Y.; Wang, J.; and Yang, Y. 2021.
\newblock Neural auto-curricula in two-player zero-sum games.
\newblock \emph{Advances in Neural Information Processing Systems}, 34: 3504--3517.

\bibitem[{Florensa et~al.(2017)Florensa, Held, Wulfmeier, Zhang, and Abbeel}]{florensa2017reverse}
Florensa, C.; Held, D.; Wulfmeier, M.; Zhang, M.; and Abbeel, P. 2017.
\newblock Reverse curriculum generation for reinforcement learning.
\newblock In \emph{Conference on robot learning}, 482--495. PMLR.

\bibitem[{Florensa et~al.(2018)}]{florensa2018automatic}
Florensa, C.; et~al. 2018.
\newblock Automatic goal generation for reinforcement learning agents.
\newblock In \emph{International conference on machine learning}, 1515--1528. PMLR.

\bibitem[{Freund and Schapire(1996)}]{freund1996game}
Freund, Y.; and Schapire, R.~E. 1996.
\newblock Game theory, on-line prediction and boosting.
\newblock In \emph{Proceedings of the ninth annual conference on Computational learning theory}, 325--332.

\bibitem[{Gruslys et~al.(2020)}]{gruslys2020advantage}
Gruslys, A.; et~al. 2020.
\newblock The advantage regret-matching actor-critic.
\newblock \emph{arXiv preprint arXiv:2008.12234}.

\bibitem[{Heinrich, Lanctot, and Silver(2015)}]{heinrich2015fictitious}
Heinrich, J.; Lanctot, M.; and Silver, D. 2015.
\newblock Fictitious self-play in extensive-form games.
\newblock In \emph{International conference on machine learning}, 805--813. PMLR.

\bibitem[{Heinrich and Silver(2016)}]{heinrich2016deep}
Heinrich, J.; and Silver, D. 2016.
\newblock Deep reinforcement learning from self-play in imperfect-information games.
\newblock \emph{arXiv preprint arXiv:1603.01121}.

\bibitem[{Hennes et~al.(2020)}]{hennes2020neural}
Hennes, D.; et~al. 2020.
\newblock Neural replicator dynamics: Multiagent learning via hedging policy gradients.
\newblock In \emph{Proceedings of the 19th International Conference on Autonomous Agents and MultiAgent Systems}, 492--501.

\bibitem[{Jaderberg et~al.(2018)}]{jaderberg2018human}
Jaderberg, M.; et~al. 2018.
\newblock Human-level performance in first-person multiplayer games with population-based deep reinforcement learning.
\newblock \emph{arXiv preprint arXiv:1807.01281}.

\bibitem[{Jin et~al.(2021)Jin, Liu, Wang, and Yu}]{jin2021v}
Jin, C.; Liu, Q.; Wang, Y.; and Yu, T. 2021.
\newblock V-Learning--A Simple, Efficient, Decentralized Algorithm for Multiagent RL.
\newblock \emph{arXiv preprint arXiv:2110.14555}.

\bibitem[{Kurach et~al.(2020)}]{kurach2020google}
Kurach, K.; et~al. 2020.
\newblock Google research football: A novel reinforcement learning environment.
\newblock In \emph{Proceedings of the AAAI Conference on Artificial Intelligence}, volume~34, 4501--4510.

\bibitem[{Lanctot et~al.(2017)Lanctot, Zambaldi, Gruslys, Lazaridou, Tuyls, P{\'e}rolat, Silver, and Graepel}]{lanctot2017unified}
Lanctot, M.; Zambaldi, V.; Gruslys, A.; Lazaridou, A.; Tuyls, K.; P{\'e}rolat, J.; Silver, D.; and Graepel, T. 2017.
\newblock A unified game-theoretic approach to multiagent reinforcement learning.
\newblock \emph{Advances in neural information processing systems}, 30.

\bibitem[{Li et~al.(2022)Li, Kong, Li, and Wu}]{li2022pmr}
Li, Y.; Kong, T.; Li, L.; and Wu, Y. 2022.
\newblock Learning Design and Construction with Varying-Sized Materials via Prioritized Memory Resets.
\newblock In \emph{2022 International Conference on Robotics and Automation (ICRA)}, 7469--7476.

\bibitem[{Littman(1994)}]{littman1994markov}
Littman, M.~L. 1994.
\newblock Markov games as a framework for multi-agent reinforcement learning.
\newblock In \emph{Proceedings of the eleventh international conference on machine learning}, volume 157, 157--163.

\bibitem[{Liu et~al.(2021)Liu, Jia, Wen, Hu, Chen, Fan, Hu, and Yang}]{liu2021towards}
Liu, X.; Jia, H.; Wen, Y.; Hu, Y.; Chen, Y.; Fan, C.; Hu, Z.; and Yang, Y. 2021.
\newblock Towards Unifying Behavioral and Response Diversity for Open-ended Learning in Zero-sum Games.
\newblock \emph{Advances in Neural Information Processing Systems}, 34: 941--952.

\bibitem[{Long et~al.(2020)}]{long2020evolutionary}
Long, Q.; et~al. 2020.
\newblock Evolutionary Population Curriculum for Scaling Multi-Agent Reinforcement Learning.
\newblock In \emph{International Conference on Learning Representations}.

\bibitem[{Lowe et~al.(2017)Lowe, Wu, Tamar, Harb, Abbeel, and Mordatch}]{lowe2017multi}
Lowe, R.; Wu, Y.; Tamar, A.; Harb, J.; Abbeel, P.; and Mordatch, I. 2017.
\newblock Multi-agent actor-critic for mixed cooperative-competitive environments.
\newblock In \emph{Proceedings of the 31st International Conference on Neural Information Processing Systems}.

\bibitem[{Matiisen et~al.(2019)Matiisen, Oliver, Cohen, and Schulman}]{matiisen2019teacher}
Matiisen, T.; Oliver, A.; Cohen, T.; and Schulman, J. 2019.
\newblock Teacher-student curriculum learning.
\newblock \emph{IEEE transactions on neural networks and learning systems}.

\bibitem[{McAleer et~al.(2021)McAleer, Lanier, Wang, Baldi, and Fox}]{mcaleer2021xdo}
McAleer, S.; Lanier, J.~B.; Wang, K.~A.; Baldi, P.; and Fox, R. 2021.
\newblock XDO: A double oracle algorithm for extensive-form games.
\newblock \emph{Advances in Neural Information Processing Systems}, 34: 23128--23139.

\bibitem[{McMahan, Gordon, and Blum(2003)}]{mcmahan2003planning}
McMahan, H.~B.; Gordon, G.~J.; and Blum, A. 2003.
\newblock Planning in the presence of cost functions controlled by an adversary.
\newblock In \emph{Proceedings of the 20th International Conference on Machine Learning (ICML-03)}, 536--543.

\bibitem[{Mehta et~al.(2020)}]{mehta2020active}
Mehta, B.; et~al. 2020.
\newblock Active domain randomization.
\newblock In \emph{Conference on Robot Learning}, 1162--1176. PMLR.

\bibitem[{Moravcik et~al.(2016)}]{moravcik2016refining}
Moravcik, M.; et~al. 2016.
\newblock Refining subgames in large imperfect information games.
\newblock In \emph{Proceedings of the AAAI Conference on Artificial Intelligence}, volume~30.

\bibitem[{Perez-Nieves et~al.(2021)Perez-Nieves, Yang, Slumbers, Mguni, Wen, and Wang}]{perez2021modelling}
Perez-Nieves, N.; Yang, Y.; Slumbers, O.; Mguni, D.~H.; Wen, Y.; and Wang, J. 2021.
\newblock Modelling behavioural diversity for learning in open-ended games.
\newblock In \emph{International Conference on Machine Learning}, 8514--8524. PMLR.

\bibitem[{Perolat et~al.(2022)}]{perolat2022mastering}
Perolat, J.; et~al. 2022.
\newblock Mastering the Game of Stratego with Model-Free Multiagent Reinforcement Learning.
\newblock \emph{arXiv preprint arXiv:2206.15378}.

\bibitem[{Portelas et~al.(2020)Portelas, Colas, Hofmann, and Oudeyer}]{portelas2020teacher}
Portelas, R.; Colas, C.; Hofmann, K.; and Oudeyer, P.-Y. 2020.
\newblock Teacher algorithms for curriculum learning of deep rl in continuously parameterized environments.
\newblock In \emph{Conference on Robot Learning}, 835--853. PMLR.

\bibitem[{Pritzel et~al.(2017)}]{pritzel2017neural}
Pritzel, A.; et~al. 2017.
\newblock Neural episodic control.
\newblock In \emph{International conference on machine learning}, 2827--2836. PMLR.

\bibitem[{Qi et~al.(2017)Qi, Yi, Su, and Guibas}]{qi2017pointnet}
Qi, C.~R.; Yi, L.; Su, H.; and Guibas, L.~J. 2017.
\newblock PointNet++: Deep Hierarchical Feature Learning on Point Sets in a Metric Space.
\newblock \emph{Advances in Neural Information Processing Systems}, 30.

\bibitem[{Racaniere et~al.(2019)Racaniere, Lampinen, Santoro, Reichert, Firoiu, and Lillicrap}]{racaniere2019automated}
Racaniere, S.; Lampinen, A.~K.; Santoro, A.; Reichert, D.~P.; Firoiu, V.; and Lillicrap, T.~P. 2019.
\newblock Automated curricula through setter-solver interactions.
\newblock \emph{arXiv preprint arXiv:1909.12892}.

\bibitem[{Schaul et~al.(2015)Schaul, Quan, Antonoglou, and Silver}]{schaul2015prioritized}
Schaul, T.; Quan, J.; Antonoglou, I.; and Silver, D. 2015.
\newblock Prioritized experience replay.
\newblock \emph{arXiv preprint arXiv:1511.05952}.

\bibitem[{Schulman et~al.(2017)Schulman, Wolski, Dhariwal, Radford, and Klimov}]{schulman2017proximal}
Schulman, J.; Wolski, F.; Dhariwal, P.; Radford, A.; and Klimov, O. 2017.
\newblock Proximal policy optimization algorithms.
\newblock \emph{arXiv preprint arXiv:1707.06347}.

\bibitem[{Schulman et~al.(2015)}]{schulman2015trust}
Schulman, J.; et~al. 2015.
\newblock Trust region policy optimization.
\newblock In \emph{International conference on machine learning}, 1889--1897. PMLR.

\bibitem[{Serrino et~al.(2019)Serrino, Kleiman-Weiner, Parkes, and Tenenbaum}]{serrino2019finding}
Serrino, J.; Kleiman-Weiner, M.; Parkes, D.~C.; and Tenenbaum, J. 2019.
\newblock Finding friend and foe in multi-agent games.
\newblock \emph{Advances in Neural Information Processing Systems}, 32.

\bibitem[{Silver et~al.(2016)}]{silver2016mastering}
Silver, D.; et~al. 2016.
\newblock Mastering the game of Go with deep neural networks and tree search.
\newblock \emph{nature}, 529(7587): 484.

\bibitem[{Steinberger, Lerer, and Brown(2020)}]{steinberger2020dream}
Steinberger, E.; Lerer, A.; and Brown, N. 2020.
\newblock DREAM: Deep regret minimization with advantage baselines and model-free learning.
\newblock \emph{arXiv preprint arXiv:2006.10410}.

\bibitem[{Szepesv{\'a}ri and Littman(1999)}]{szepesvari1999unified}
Szepesv{\'a}ri, C.; and Littman, M.~L. 1999.
\newblock A unified analysis of value-function-based reinforcement-learning algorithms.
\newblock \emph{Neural computation}, 11(8): 2017--2060.

\bibitem[{Tang et~al.(2021)}]{tang2021discovering}
Tang, Z.; et~al. 2021.
\newblock Discovering diverse multi-agent strategic behavior via reward randomization.
\newblock \emph{arXiv preprint arXiv:2103.04564}.

\bibitem[{Tesauro(1995)}]{tesauro1995temporal}
Tesauro, G. 1995.
\newblock Temporal difference learning and TD-Gammon.
\newblock \emph{Communications of the ACM}, 38(3): 58--68.

\bibitem[{Vinyals et~al.(2019)}]{vinyals2019grandmaster}
Vinyals, O.; et~al. 2019.
\newblock Grandmaster level in {StarCraft II} using multi-agent reinforcement learning.
\newblock \emph{Nature}, 575(7782): 350--354.

\bibitem[{Wang et~al.(2019)Wang, Lehman, Clune, and Stanley}]{wang2019poet}
Wang, R.; Lehman, J.; Clune, J.; and Stanley, K.~O. 2019.
\newblock Poet: open-ended coevolution of environments and their optimized solutions.
\newblock In \emph{Proceedings of the Genetic and Evolutionary Computation Conference}, 142--151.

\bibitem[{Wang et~al.(2020{\natexlab{a}})Wang, Lehman, Rawal, Zhi, Li, Clune, and Stanley}]{wang2020enhanced}
Wang, R.; Lehman, J.; Rawal, A.; Zhi, J.; Li, Y.; Clune, J.; and Stanley, K. 2020{\natexlab{a}}.
\newblock Enhanced POET: Open-ended reinforcement learning through unbounded invention of learning challenges and their solutions.
\newblock In \emph{International Conference on Machine Learning}, 9940--9951. PMLR.

\bibitem[{Wang et~al.(2020{\natexlab{b}})Wang, Yang, Liu, Hao, Hao, Hu, Chen, Fan, and Gao}]{wang2020few}
Wang, W.; Yang, T.; Liu, Y.; Hao, J.; Hao, X.; Hu, Y.; Chen, Y.; Fan, C.; and Gao, Y. 2020{\natexlab{b}}.
\newblock From few to more: Large-scale dynamic multiagent curriculum learning.
\newblock In \emph{Proceedings of the AAAI Conference on Artificial Intelligence}, volume~34, 7293--7300.

\bibitem[{Warde-Farley et~al.(2019)Warde-Farley, de~Wiele, Kulkarni, Ionescu, Hansen, and Mnih}]{warde2018unsupervised}
Warde-Farley, D.; de~Wiele, T.~V.; Kulkarni, T.; Ionescu, C.; Hansen, S.; and Mnih, V. 2019.
\newblock Unsupervised Control Through Non-Parametric Discriminative Rewards.
\newblock In \emph{International Conference on Learning Representations}.

\bibitem[{Watkins and Dayan(1992)}]{watkins1992q}
Watkins, C.~J.; and Dayan, P. 1992.
\newblock Q-learning.
\newblock \emph{Machine learning}, 8(3): 279--292.

\bibitem[{Wen et~al.(2022)}]{wen2022multi}
Wen, M.; et~al. 2022.
\newblock Multi-agent reinforcement learning is a sequence modeling problem.
\newblock \emph{Advances in Neural Information Processing Systems}, 35: 16509--16521.

\bibitem[{Yu et~al.(2021)Yu, Velu, Vinitsky, Wang, Bayen, and Wu}]{yu2022mappo}
Yu, C.; Velu, A.; Vinitsky, E.; Wang, Y.; Bayen, A.; and Wu, Y. 2021.
\newblock The surprising effectiveness of ppo in cooperative, multi-agent games.
\newblock \emph{arXiv preprint arXiv:2103.01955}.

\bibitem[{Zhang and Sandholm(2021)}]{zhang2021subgame}
Zhang, B.; and Sandholm, T. 2021.
\newblock Subgame solving without common knowledge.
\newblock \emph{Advances in Neural Information Processing Systems}, 34: 23993--24004.

\bibitem[{Zhang, Abbeel, and Pinto(2020)}]{zhang2020automatic}
Zhang, Y.; Abbeel, P.; and Pinto, L. 2020.
\newblock Automatic curriculum learning through value disagreement.
\newblock \emph{Advances in Neural Information Processing Systems}, 33: 7648--7659.

\bibitem[{Zinkevich et~al.(2007)Zinkevich, Johanson, Bowling, and Piccione}]{zinkevich2007regret}
Zinkevich, M.; Johanson, M.; Bowling, M.; and Piccione, C. 2007.
\newblock Regret minimization in games with incomplete information.
\newblock \emph{Advances in neural information processing systems}, 20.

\end{thebibliography}

\clearpage
\setcounter{secnumdepth}{3}
\setcounter{section}{0}
\renewcommand{\thesection}{\Alph{section}}
\section{Analysis and Proofs}\label{sec:app}

\subsection{Detailed Analysis of the Motivating Example}\label{sec:app:example}

We first show that the entire state space of $RPS(n)$ can be covered within $\mathcal{O}(n)$ samples by using a state buffer and resetting games to the newly visited states.
We start with an empty state buffer, and the game resets according to its initial state distribution $\rho(\cdot)$, which always resets the game to $s_0$.
With a random exploration policy, the probability for the game to transit from $s_0$ to $s_1$ is $1/3$.
Therefore, the number of samples required to visit state $s_1$ in expectation is $\mathbb{E}[n(s_1)] = 3$.
After $s_1$ is visited, this new state will be stored in the state buffer.
Since we select the newly visited states as the initial state, the game will be reset to state $s_1$, and the additional number of samples required to visit state $s_2$ in expectation is also $\mathbb{E}[n(s_2)] = 3$.
In general, by starting from state $s_{k-1}$, the expected number of samples to visit state $s_k$ is $\mathbb{E}[n(s_k)] = 3$, $k = 1, 2, \cdots, n - 1$.
Therefore, the total number of samples required to cover the whole state space is $\sum_{k=1}^{n-1} \mathbb{E}[n(s_k)] = 3(n - 1)$, which is $\mathcal{O}(n)$.

Given that the state buffer has covered the entire state space, we then show that the NE Q-value of $RPS(n)$ can be learned by solving subgames with minimax-Q backward from $RSP(1)$ to $RPS(n)$.
Consider using minimax-Q to solve $RPS(1)$, we can set the learning rate $\alpha=1$ since the transition is deterministic, and the NE Q-value of a state-action pair $(s, \bm{a})$ can be learned when this pair is in the collected samples.
Therefore, to learn the NE Q-values of $RPS(1)$, we have to collect all state-action pairs at least one time.
With a random exploration policy, the number of samples required to cover all state-action pairs is $\sum_{i=1}^9 9/i = 25.46 < 26$.
Therefore, the NE Q-values of $RPS(1)$ can be learned within $26$ samples in expectation.
Given that the NE Q-values of $RPS(k)$ are learned, the NE Q-values of $RPS(k + 1)$ are only wrong at the first state and can be learned within $26$ episodes in expectation.
Note that the expected episode length of $RPS(\infty)$ is $1.5$, so the expected episode length of $RPS(k)$ is less than $1.5$.
Consider the episode used to learn the NE Q-values of the first state of $RPS(k + 1)$, either $P_1$ wins and the expected episode length is less than $1 + 1.5=2.5$, or $P_1$ draws or loses and the episode length is $1$.
In both cases, the episode length is less than $2.5$, so the number of samples used is less than $26*2.5=65$.
Therefore, the total number of samples used to learn the NE Q-values from $RPS(1)$ to $RPS(n)$ is less than $65(n - 1)$, which is $\mathcal{O}(n)$.

Since it takes $\mathcal{O}(n)$ samples to cover the entire state space and $\mathcal{O}(n)$ samples to learn the NE Q-values from $RPS(1)$ to $RPS(n)$, the total number of samples is still $\mathcal{O}(n)$.

\subsection{Proof of Proposition~\ref{prop:correctness}}\label{sec:app:proof}

\correctness*
\begin{proof}
When the policy trained by subgame curriculum learning converges, it is an NE of all subgames induced by the proposed states, including all initial states $s_0$ with $\rho(s_0) > 0$. 
Therefore, it is an NE of the original Markov game.
\end{proof}

\subsection{Detailed Analysis of the State Sampling Metric}\label{sec:app:metric}

We approximate the squared difference between the current value and the NE value by Eq.~(\ref{eq:estimate}), i.e.
\begin{align}
    w(s) &= \mathbb{E}_i\big[(V_1^*(s) - \tilde{V}_i(s))^2\big]\nonumber \\
    &\approx \alpha \cdot \mathbb{E}_i\big[\tilde{V}_i^{(t)}(s) - \tilde{V}^{(t-1)}_i(s)\big]^2 + \mathrm{Var}_i\big[\tilde{V}_i(s)\big]. \nonumber
\end{align}
The first term in Eq.~(\ref{eq:estimate}) uses a hyperparameter $\alpha$ and the difference between two consecutive value function checkpoints to estimate the difference between the current value and the NE value.
As shown in Fig.~\ref{fig:app:bias_estimate}, when the value function changes monotonically throughout training, the estimate can be regarded as a first-order approximation of the bias term.
However, the value function of zero-sum games may oscillate up and down in different emergent stages (like in hide-and-seek), as shown in Fig.~\ref{fig:app:oscillate}.
In this case, the difference between two value function checkpoints is no longer an approximation of the distance to the NE value but a first-order approximation of the difference between the current value and the next local minimal or local maximal value $V_1^{(*, k)}$, and the weight becomes the approximated squared difference between the current value and the next local optimal value, i.e.,
\begin{align}    
    w(s) &= \alpha \cdot \mathbb{E}_i\big[\tilde{V}_i^{(t)}(s) - \tilde{V}^{(t-1)}_i(s)\big]^2 + \mathrm{Var}_i\big[\tilde{V}_i(s)\big] \nonumber \\
    &\approx \mathbb{E}_i\big[V_1^{(*, k)}(s) - \tilde{V}_i(s)\big]^2 + \mathrm{Var}_i\big[V_1^{(*, k)}(s)-\tilde{V}_i(s)\big] \nonumber \\
    &= \mathbb{E}_i\big[(V_1^{(*, k)}(s) - \tilde{V}_i(s))^2\big].
\end{align}
Therefore, by using the weight in Eq.~(\ref{eq:estimate}), we are not directly prioritizing states where the values are far from the NE values but prioritizing states where the values are far from the next local optimal value.
For example, in Fig.~\ref{fig:app:oscillate}, before the value function has learned the first local maximal value $V_1^{(*, 1)}$, we will give larger weights to states that are far from the $V_1^{(*, 1)}$ to accelerate the first stage of learning $V_1^{(*, 1)}$.
After $V_1^{(*, 1)}$ is successfully learned, we will then prioritize states that are far from the second local optimal value $V_2^{(*, 1)}$ and accelerate the second stage of learning $V_1^{(*, 2)}$.
Finally, we learn towards the NE value $V_1^{(*, 3)} = V_1^*$.
By accelerating the learning in each stage, we make the NE learning process more efficient in total.

It is also possible to train an ensemble of value functions for each player to improve the estimation.
Suppose we train $M$ value functions for player $i$ and denote them as $\{\tilde{V}_{i, m}\}_{m=1}^M$ for $i=1, 2$, then the weight for state $s$ becomes
\begin{align}    
    w(s) &= \alpha \cdot \mathbb{E}_{i, m}\big[\tilde{V}_{i, m}^{(t)}(s) - \tilde{V}^{(t-1)}_{i, m}(s)\big]^2 + \mathrm{Var}_{i, m}\big[\tilde{V}_{i, m}(s)\big],
\end{align}
where the expectation and variance are taken over both the player index $i$ and the ensemble index $m$.


\begin{figure}[t]
\centering
\includegraphics[width=0.8\linewidth]{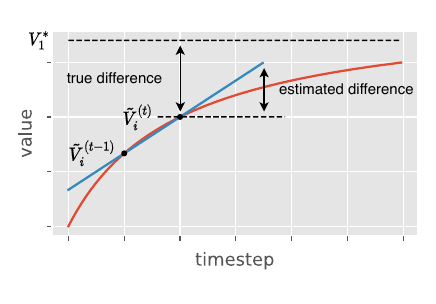}
\vspace{-2em}
\caption{Approximation of the bias term when value function changes monotonically.}
\label{fig:app:bias_estimate}
\vspace{-1em}
\end{figure}

\begin{figure}[t]
\centering
\includegraphics[width=0.8\linewidth]{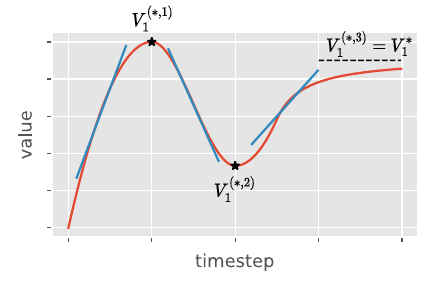}
\vspace{-2em}
\caption{Approximation in different stages when value function oscillates in training.}
\label{fig:app:oscillate}
\end{figure}

\subsection{Difference of {\name} and subgame solving method for extensive-form games}
First, we would like to emphasize that the goal of this work is to accelerate learning in complex fully-observable Markov games. In our experiments, the learning agents do not know the transition of the games, following the standard assumption in reinforcement learning.

For extensive-form games, such as poker, there has been extensive literature on how to construct and solve a subgame~\citep{zhang2021subgame}~\citep{brown2017safe}~\citep{burch2014solving}~\citep{moravcik2016refining}. The idea of subgame solving is first to get a blueprint strategy of the abstracted game and use it to play the original game. As the game progresses and the remaining game becomes tractable, the specific subgame is solved in real-time to create a combined final policy. Subgame solving typically uses iterative updates based on regret matching to find the policy, which requires the traverse of the game tree. 

\section{Implementation Details}\label{app:impl}

\subsection{Implementation of Farthest Point Sampling}\label{sec:app:fps}
In the subgame sampler, we use a state buffer to approximate the whole state space and record the state weights. In principle, the states in the buffer should span the entire state space and distribute uniformly, but the rollout data is usually concentrated and very similar to each other. Therefore, we need to select states that are sufficiently far from each other to ensure good coverage of the state space. Formally, we need to select a subset $S'$ of size $K$ from the whole set $S$ so that the sum of the shortest distances between states in the subset $S'$ is maximized, i.e., $max_{S'\subset S, |S'|=K} \Sigma_{s\in S'} min_{s'\in S'}|s-s'|$. The farthest point sampling is a greedy algorithm that efficiently finds an approximate optimal solution to this problem.

In general, FPS iteratively selects the farthest point from the current set of points. The distance between two states is simply the Euclidean distance. The distance between a state $s_a$ and a set of states $S$ is the smallest distance between $s_a$ and any state in $S$, i.e., $min_{s\in S}|s_a-s|$. For implementation, we first normalize each dimension of the state vector to make all values lie in the range $[0,1]$. Then, we directly use the \emph{farthest\_point\_sampler()} function from the Deep Gragh Library to utilize GPUs for fast and stable results. 

\subsection{Training Details}\label{sec:app:training}
\begin{figure}
    \centering
    \includegraphics[width=0.8\linewidth]{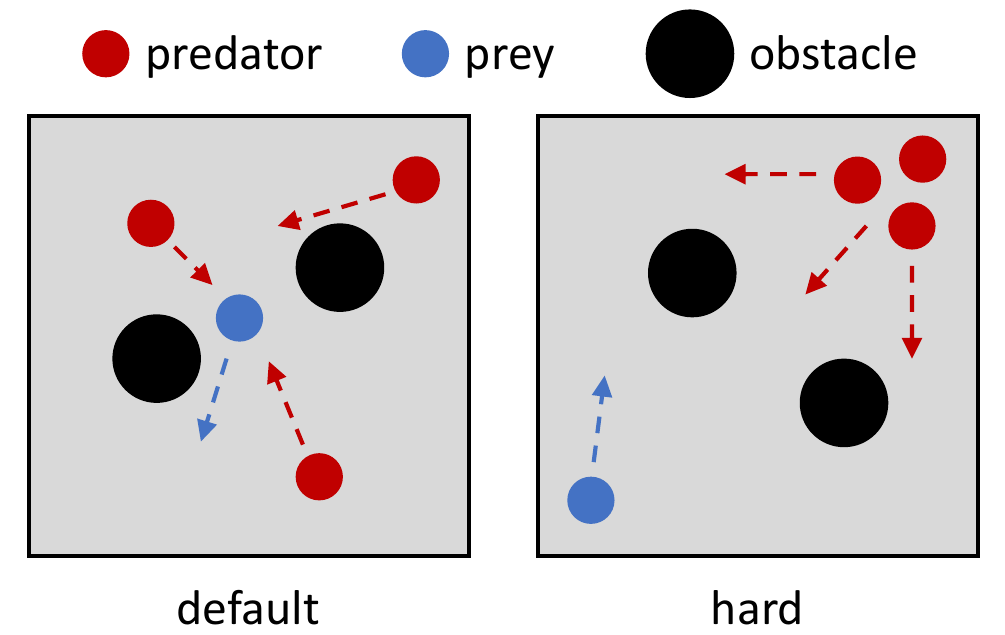}
    \caption{Illustration of the default and hard setting of predator-prey in MPE.}
    \label{fig:app:predator_prey}
\end{figure}

\textbf{Multi-Agent Particle Environment.}
The default and hard settings of the predator-prey scenario in MPE are shown in Fig.~\ref{fig:app:predator_prey}.
The environment is a 2D square space, and the length of a side is 4, i.e., $\{(x, y)| -2 \leq x \leq 2, -2 \leq y \leq 2\}$.
Three predators (red) cooperatively chase one prey (blue), and there are two obstacles in the space.
In the default setting, all agents and obstacles are randomly spawned.
In the hard setting, predators are uniformly spawned in the top-right corner, i.e., $\{(x, y)| 1 \leq x \leq 2, 1 \leq y \leq 2\}$, the prey is spawned in the bottom-left corner, i.e., $\{(x, y)| -2 \leq x \leq -1, -2 \leq y \leq -1\}$, and the obstacles are still randomly generated in the square.

This environment is fully observable, and the state of each agent is a concatenation of the positions and the velocities of all agents and the positions of all obstacles.
The action space is discrete with $5$ actions: idle, up, down, left, right.
The environment lasts for $200$ steps. 
In each step, if any predator collides with the prey, all predators get a reward of $+1$, and the prey gets a reward of $-1$.

The actor and critic networks use the transformer architecture. The inputs first pass through a LayerNorm layer. The normalized states are divided into different entities, including self, other agents, obstacles, and time; then, each entity passes through fully connected layers to get its embedding. The weights of the embedding layers are shared within entities of the same type. Then, the embedding of each entity is concatenated with the self-states and passed through a self-attention network. Then, we average the output of the attention block and concatenate it with the self-embedding to get the final representation. This representation is then passed through a LayerNorm layer and an MLP layer and then produces the value through a critic's head and the action through an actor's head. All hyperparameters for training are listed in Table~\ref{tab:app:mpe}.

\begin{figure}[t]
\begin{minipage}{0.45\textwidth}
\centering
\begin{tabular}{lc}
    \toprule
    Hyperparameters & Value\\
    \midrule
    Learning rate & 5e-4\\
    Discount rate ($\gamma$) & 0.99\\
    GAE parameter ($\lambda_{\textrm{GAE}}$) & 0.95\\
    Gradient clipping & 10.0\\
    Adam stepsize  & 1e-5\\
    Value loss coefficient & 1\\
    Entropy coefficient & 0.01\\
    Parallel threads & 100\\
    PPO clipping & 0.2\\
    PPO epochs & 5\\
    Size of embedding layer &32\\
    Size of MLP layer &64\\
    Size of LSTM layer &64\\
    Residual attention layer &8\\
    probability $p$ & 0.7 \\
    Ensemble size $M$ & 3 \\
    Capacity $K$ & 10000 \\
    Weight of the value difference $\alpha$ & 0.7 \\
    \bottomrule
\end{tabular}
\captionof{table}{Hyperparameters of MPE.}
\label{tab:app:mpe}
\end{minipage}
\hfill
\begin{minipage}{0.45\textwidth}
\centering
\begin{tabular}{lc}
    \toprule
    Hyperparameters & Value\\
    \midrule
    Learning rate & 5e-4\\
    Discount rate ($\gamma$) & 0.99\\
    GAE parameter ($\lambda_{\textrm{GAE}}$) & 0.95\\
    Gradient clipping & 10.0\\
    Adam stepsize  & 1e-5\\
    Value loss coefficient & 1\\
    Entropy coefficient & 0.01\\
    Parallel threads & 200\\
    PPO clipping & 0.2\\
    PPO epochs & 10\\
    Size of MLP layer & 64\\
    probability $p$ & 0.7 \\
    Ensemble size $M$ & 3 \\
    Capacity $K$ & 10000 \\
    Weight of the value difference $\alpha$ & 0.7 \\
    \bottomrule
\end{tabular}
\captionof{table}{Hyperparameters of GRF.}
\label{tab:app:grf}
\end{minipage}
\end{figure}

\begin{table}[t]
\centering
\begin{tabular}{cc}
    \toprule
    Length & Information\\
    \midrule
    22 & (x,y) coordinates of left team players\\
    22 & (x,y) direction of left team players \\
    22 & (x,y) coordinates of right team players\\
    22 & (x,y) direction of right team players \\
    3 & (x, y and z) ball position \\
    3 & ball direction \\
    3 & one hot encoding of ball ownership (none, left, right) \\
    11 & one hot encoding of which player is active \\
    7 & one hot encoding of game mode \\
    \bottomrule
\end{tabular}
\captionof{table}{Information in the state vector of GRF.}
\label{tab:app:grf_obs}
\end{table}

\begin{table}[t]
\centering
\begin{tabular}{lc}
    \toprule
    Hyper-parameters&Value\\
    \midrule
    Learning rate & 3e-4\\
    Discount rate ($\gamma$) & 0.998\\
    GAE parameter ($\lambda_{\textrm{GAE}}$) & 0.95\\
    Gradient clipping & 5.0\\
    Adam stepsize  & 1e-5\\
    Value loss coefficient & 1\\
    Entropy coefficient & 0.01\\
    PPO clipping & 0.2\\
    Chunk length & 10\\
    PPO epochs & 4\\
    Horizon & 80 \\
    Mini-batch size & 64000\\
    Size of embedding layer &128\\
    Size of MLP layer &256\\
    Size of LSTM layer &256\\
    Residual attention layer &32\\
    Weight decay coefficient &$10^{-6}$ \\
    probability $p$ & 0.7 \\
    Ensemble size $M$ & 3 \\
    Capacity $K$ & 10000 \\
    Weight of the value difference $\alpha$ & 1.0 \\
    \bottomrule
\end{tabular}
\captionof{table}{Hyperparameters of HnS.}
\label{tab:app:hns}
\end{table}

\begin{table}[t]
\centering
\begin{tabular}{lc}
    \toprule
    Hyperparameters & Value\\
    \midrule
    Learning rate & 5e-4\\
    Discount rate ($\gamma$) & 0.99\\
    GAE parameter ($\lambda_{\textrm{GAE}}$) & 0.95\\
    Gradient clipping & 20.0\\
    Adam stepsize  & 1e-5\\
    Value loss coefficient & 1\\
    Entropy coefficient & 0.01\\
    PPO clipping & 0.2\\
    chunk length & 10\\
    PPO epochs & 15\\
    Horizon & 60 \\
    Parallel threads & 300 \\
    probability $p$ & 0.7 \\
    Ensemble size $M$ & 3 \\
    Capacity $K$ & 2000 \\
    \bottomrule
\end{tabular}
\captionof{table}{Hyperparameters of the cooperative task in HnS}
\label{tab:app:hns_coop}
\end{table}

\begin{figure*}[t]
\centering
\begin{minipage}{0.24\textwidth}
\begin{subfigure}
    \centering
    \includegraphics[width=\textwidth,trim={8cm 0 8cm 0},clip]{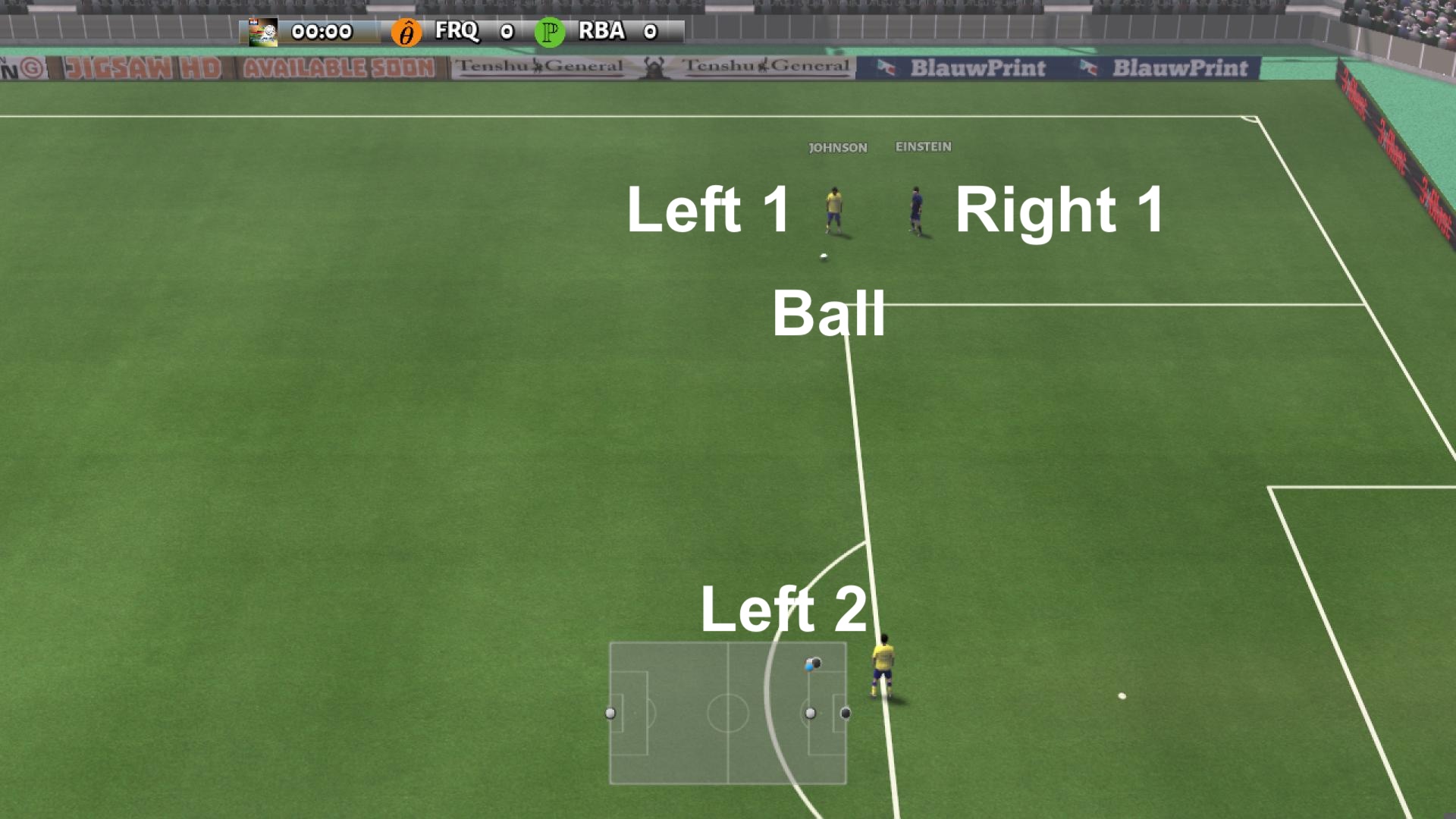}
    \caption{Pass and shoot scenario in GRF.}
    \label{fig:app:run_pass}
\end{subfigure}
\end{minipage}
\hfill
\begin{minipage}{0.24\textwidth}
\begin{subfigure}
    \centering
    \includegraphics[width=\textwidth,trim={8cm 0 8cm 0},clip]{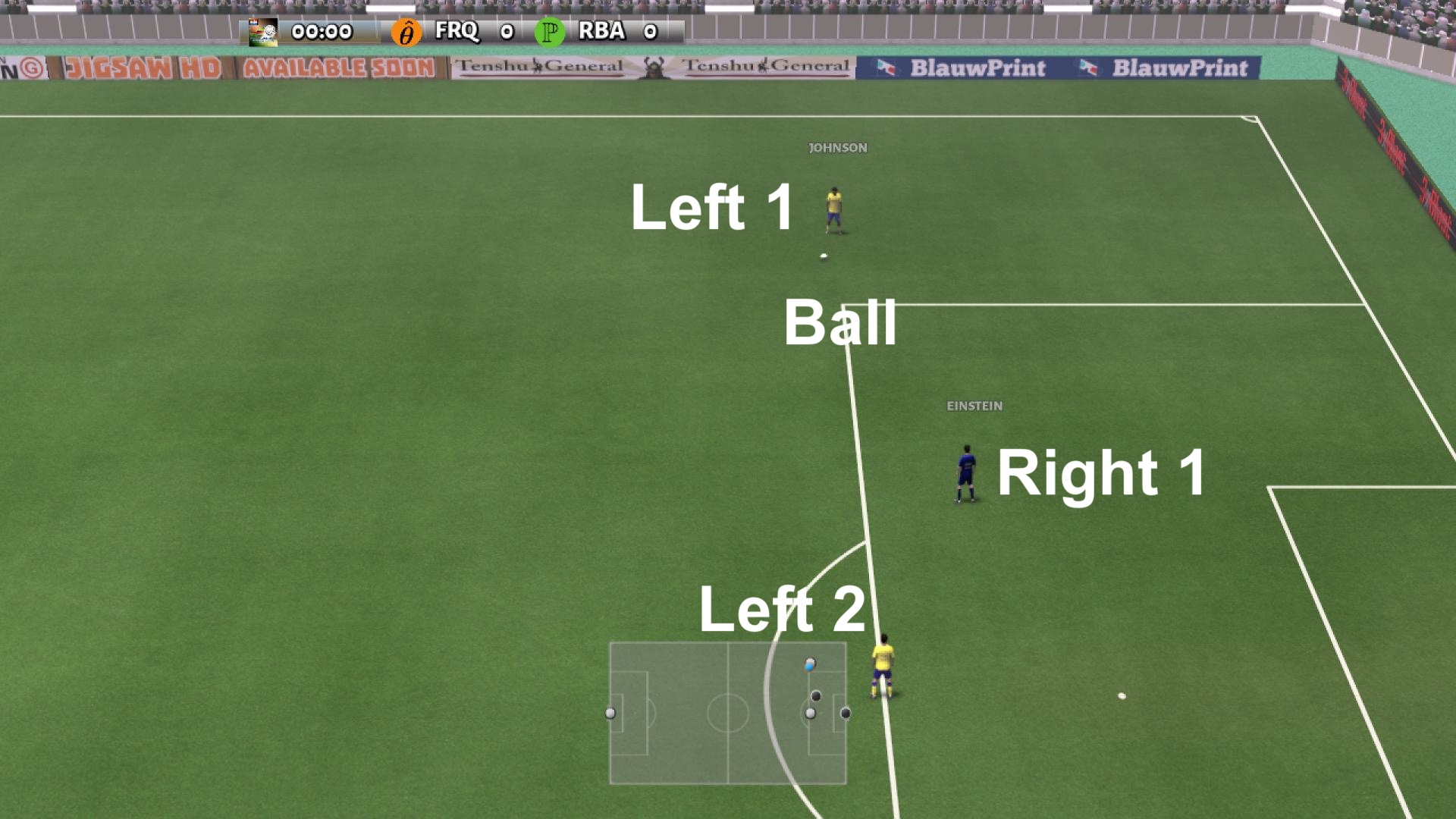}
    \caption{Run, pass and shoot scenario in GRF.}
    \label{fig:app:run_pass_shoot}
\end{subfigure}
\end{minipage}
\hfill
\begin{minipage}{0.24\textwidth}
\begin{subfigure}
    \centering
    \includegraphics[width=\textwidth,trim={8cm 0 8cm 0},clip]{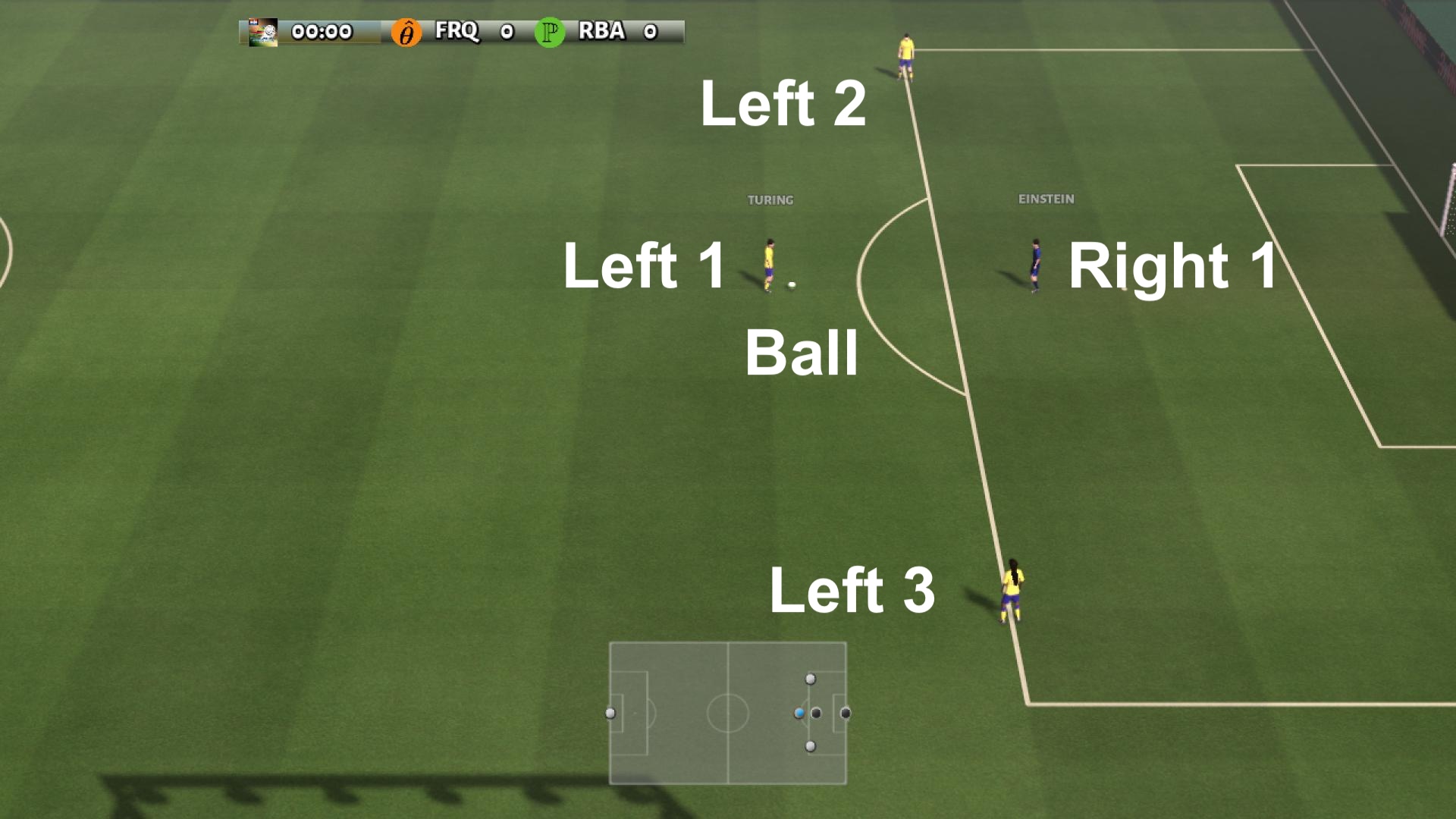}
    \caption{3 vs 1 with keeper scenario in GRF.}
    \label{fig:app:3vs1}
\end{subfigure}
\end{minipage}
\hfill
\begin{minipage}{0.18\textwidth}
\begin{subfigure}
    \centering
    \includegraphics[width=\textwidth]{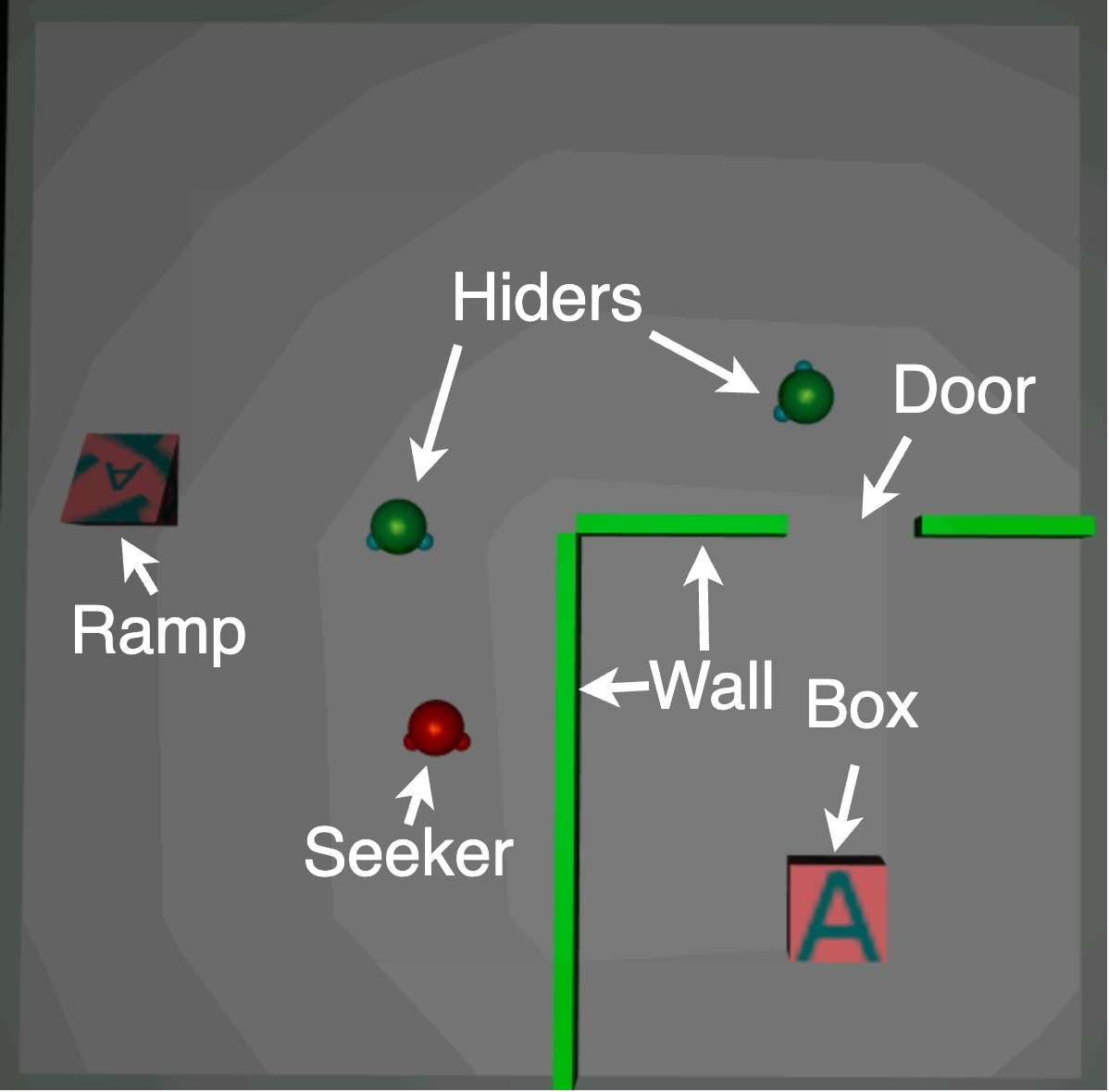}
    \caption{Quadrant scenario in HnS.}
    \label{fig:app:quadrant}
\end{subfigure}
\end{minipage}
\end{figure*}

\textbf{Google Research Football.}
The environment is a physics-based 3D football simulation, and the length and width are 2.0 and 0.9, i.e., $\{(x, y)| -1.0 \leq x \leq 1.0, -0.45 \leq y \leq 0.45\}$. The \textit{pass and shoot} scenario in GRF is shown in Fig.~\ref{fig:app:run_pass}. There are five players and a soccer ball in the environment, with a scripted goalkeeper and two RL attackers on the left side and a scripted goalkeeper and one RL defender on the right side. The left goalkeeper is spawned at $(-1.0, 0.0)$ and the two attackers are spawned at $(0.7,0.0)$ and $(0.7, -0.3)$. The right goalkeeper is spawned at $(1.0, 0.0)$ and the defender is spawned at $(0.75, -0.3)$. The ball is spawned at $(0.7, -0.28)$. The \textit{run, pass and shoot} scenario in GRF is shown in Fig.~\ref{fig:app:run_pass_shoot}. There are five players and a soccer ball in the environment, with a scripted goalkeeper and two RL attackers on the left side and a scripted goalkeeper and one RL defender on the right side. The left goalkeeper is spawned at $(-1.0, 0.0)$ and the two attackers are spawned at $(0.7,0.0)$ and $(0.7, -0.3)$. The right goalkeeper is spawned at $(1.0, 0.0)$ and the defender is spawned at $(0.75, -0.1)$. The ball is spawned at $(0.7, -0.28)$. The \textit{3 vs 1 with keeper} scenario in GRF is shown in Fig.~\ref{fig:app:3vs1}. There are six players and a soccer ball in the environment, with a scripted goalkeeper and three RL attackers on the left side and a scripted goalkeeper and one RL defender on the right side. The left goalkeeper is spawned at $(-1.0, 0.0)$ and the three attackers are spawned at $(0.6,0.0)$, $(0.7, 0.2)$ and $(0.7, -0.2)$. The right goalkeeper is spawned at $(1.0, 0.0)$ and the defender is spawned at $(0.75, 0.0)$. The ball is spawned at $(0.6, 0.0)$. In all three environments, attackers have to learn how to dribble the ball, cooperate with teammates to pass the ball, and overcome the defender’s defense to score goals.

The environment is fully observable, and the state of each agent is a 115-dimensional vector, including the coordinates of left team players, the directions of left team players, the coordinates of right team players, the directions of right team players, the ball position, the ball direction, one hot encoding of ball ownership, one hot encoding of which player is active and one hot encoding of game mode. The detailed information is listed in Table~\ref{tab:app:grf_obs}. The action space is discrete with 19 actions: idle, left, top left, top, top right, right, bottom right, bottom, bottom left, long pass, high pass, short pass, shoot, start sprinting, reset current movement direction, stop sprinting, slide, start dribbling and stop dribbling. An episode lasts a maximum of 200 steps. The environment ends prematurely when one side scores, the possession of the ball changes, or the game is out of play. We use the standard scoring and checkpoint rewards provided by the football engine. Specifically, if the left team scores a goal in each step, all left players get a reward of +1, and the right player gets -1.
There are also ten concentric circles with the goal in the center, called checkpoint regions.
The left team obtains an additional checkpoint reward of +0.1 when they possess the ball, and first move into the next checkpoint region, and the right team gets -0.1.
Checkpoint rewards are only given once per episode. 

The inputs of the actor and critic networks first pass through a LayerNorm layer. The normalized states then pass through an MLP layer and then produce the value through a critic head and the action through an actor's head. All hyperparameters for training are listed in Table~\ref{tab:app:grf}.

\begin{figure}[t]
    \centering
    \includegraphics[width=0.7\linewidth]{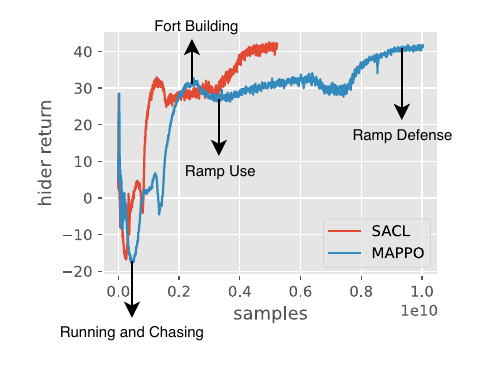}
    \vspace{-3mm}
    \caption{Checkpoints of four rounds of emergent strategies in HnS.}
    \label{fig:exp:hns_appendix}
\end{figure}

\textbf{Hide-and-seek Environment.}
\begin{figure*}[t]
\centering
\includegraphics[width=0.8\linewidth]{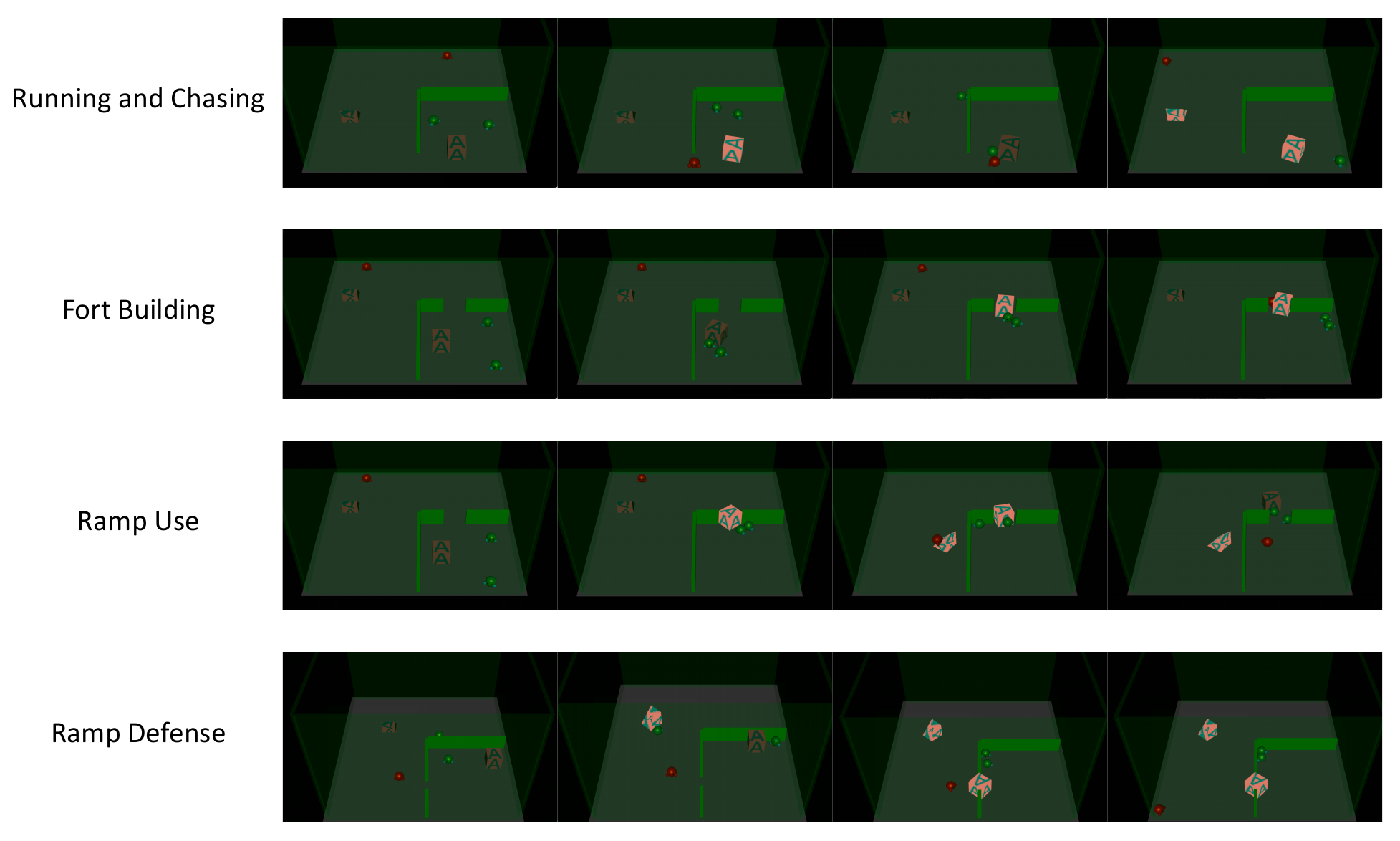}
\caption{Sample trajectory traces from each emergent stage in quadrant scenario of HnS.}
\label{fig:app:behaviors}
\end{figure*}

The quadrant scenario in the hide-and-seek environment is shown in Fig.~\ref{fig:app:quadrant}.
The environment is a square space with a square room with a door in the bottom-right corner.
There are two hiders (green), one seeker (red), one box, and one ramp.
At the beginning of each episode, the hiders, box, and ramp are uniformly spawned inside the room, and the seeker is uniformly spawned outside the room.

The environment is fully observable, and the state of each agent is a concatenation of the positions and velocities of all agents, the positions, velocities, and lock flags of the box and the ramp, and the current timestep.
The action space is discrete, and agents can choose to move in 4 directions: grab and lock/unlock.
Each episode lasts for $80$ steps and is divided into 2 phases: the preparation phase and the main phase.
In the preparation phase, the seeker is fixed, and only the hiders can act to prepare for the main phase.
No reward is given to any agent in the preparation phase.
In the main phase, all agents can act, and the seeker tries to find the hiders, and the hiders try to avoid being discovered.
When the seeker spots the hiders, the seeker gets a reward of $+1$ at this step, and the hiders get a reward of $-1$. 
Otherwise, the seeker gets a reward of $-1$, and the hiders get $+1$.

There are a total of 4 emergent stages in this game, as shown in Fig.~\ref{fig:app:behaviors}.
(1) \emph{Running and Chasing}: The hiders learn to run away from the seeker to avoid detection, while the seeker learns to chase the hiders. The seeker is the winner at this stage, and the average episode reward of hiders is about $-20$.
(2) \emph{Fort Building}: In the preparation phase, the hiders learn to use the box to block the door and lock it in place to build a fort so that the seeker cannot enter the room and see the hider. The hiders are the winners in this stage, and the average episode reward of hiders is about $30$.
(3) \emph{Ramp Use}: The seeker learns to move the ramp to the wall of the room and use it to get into the room. The average episode reward of hiders reduces to about $25$ but is still larger than $0$.
(4) \emph{Ramp Defense}: In the preparation phase, the hiders learn to move the ramp into the room or push it far away from the wall and lock it to prevent being used by the seeker. The seeker can no longer enter the room and find the hiders. The average episode reward of hiders is about $40$ at this stage.

We adopt the same network architecture as~\citep{baker2020emergent}.
The states are divided into different entities, including self, other agents, box, and ramp; then, each entity passes through fully connected layers to get its embedding.
The weights of the embedding layers are shared within entities of the same type.
Then, the embedding of each entity is concatenated with the self embedding and passed through a self-attention network.
Then, we average the output of the attention block and concatenate it with the self-embedding to get the final representation.
This representation is then passed through an MLP layer and an LSTM layer and then produces the value through a critic head and the action through an actor's head.
All hyperparameters of HnS are listed in Table~\ref{tab:app:hns}.

Besides zero-sum games, it is also possible to use {\name} in cooperative tasks.
We choose the Ramp Use stage in HnS to show that {\name} can produce comparable results to curriculum learning algorithms specialized for cooperative tasks~\citep{chen2021variational}.
In this task, there is one hider with a fixed policy, one seeker to train, one box, and one ramp.
We need to train a seeker policy to use the ramp to get into the quadrant
room for positive rewards. 
The environment is fully observable, and the state is the same as that in the quadrant scenario.
We use the same prior knowledge to define easy tasks as ~\citep{chen2021variational}, which prioritizes states where the ramp is right next to the wall, and agents are next to the ramp. 
All hyperparameters are listed in Table~\ref{tab:app:hns_coop}.

\begin{figure*}[t]
\centering
\subfigure[MPE.]
{
    \includegraphics[width=0.4\textwidth]{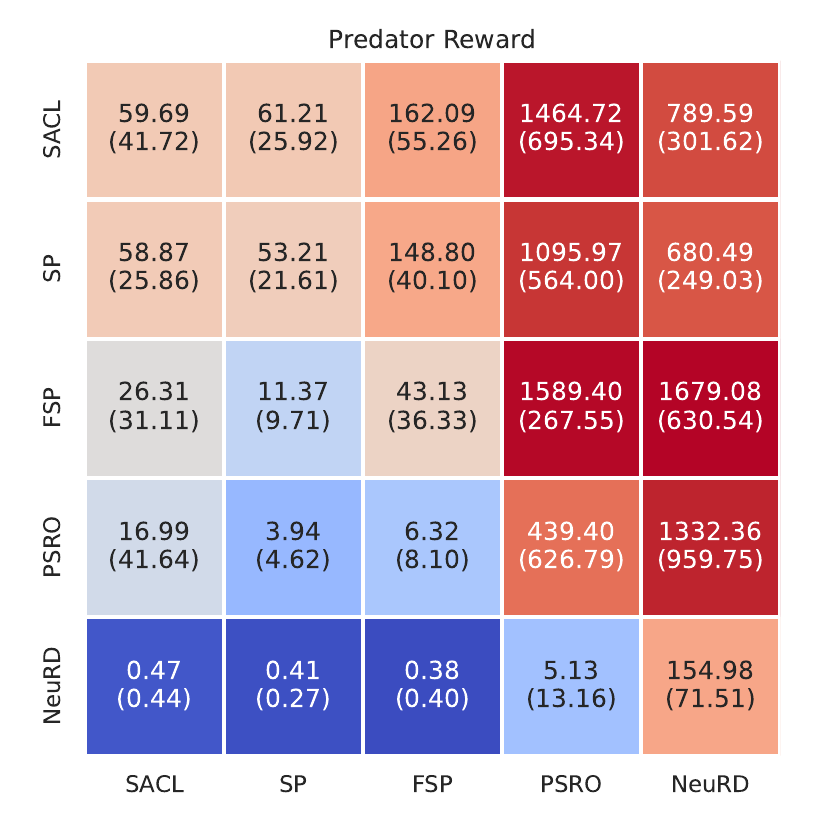}
    \label{fig:app:cross_easy}
}
\subfigure[MPE hard.]
{
    \includegraphics[width=0.4\textwidth]{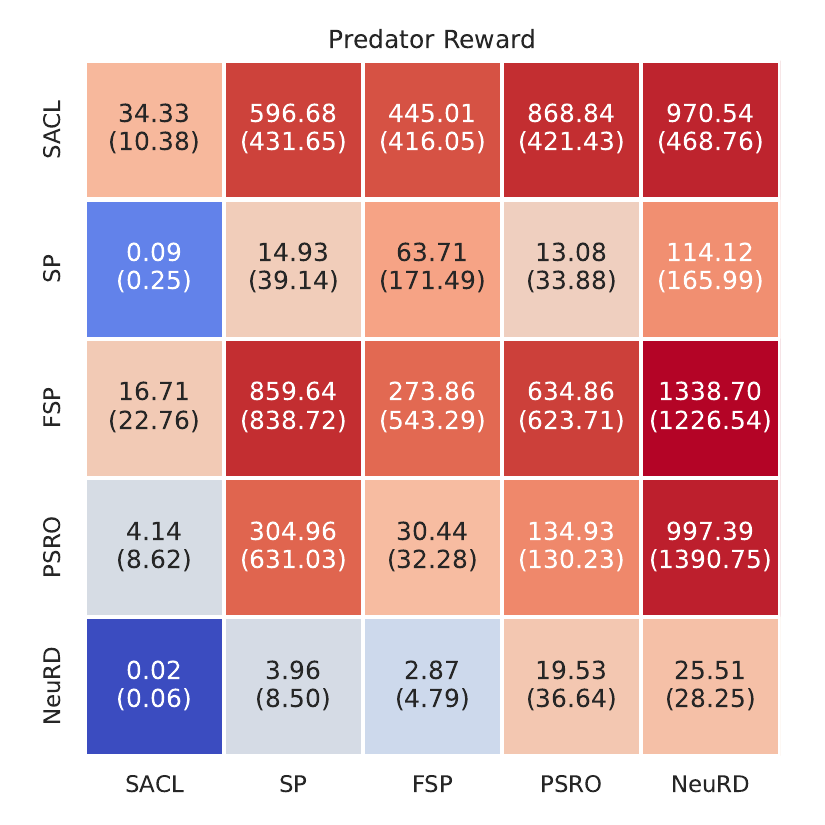}
    \label{fig:app:cross_hard}
}
\vspace{-3mm}
\caption{Cross-play results in MPE and MPE hard.}
\vspace{-2mm}
\end{figure*}

\begin{figure*}
\centering
\subfigure[5M: Center.]
{
    \includegraphics[width=0.25\textwidth]{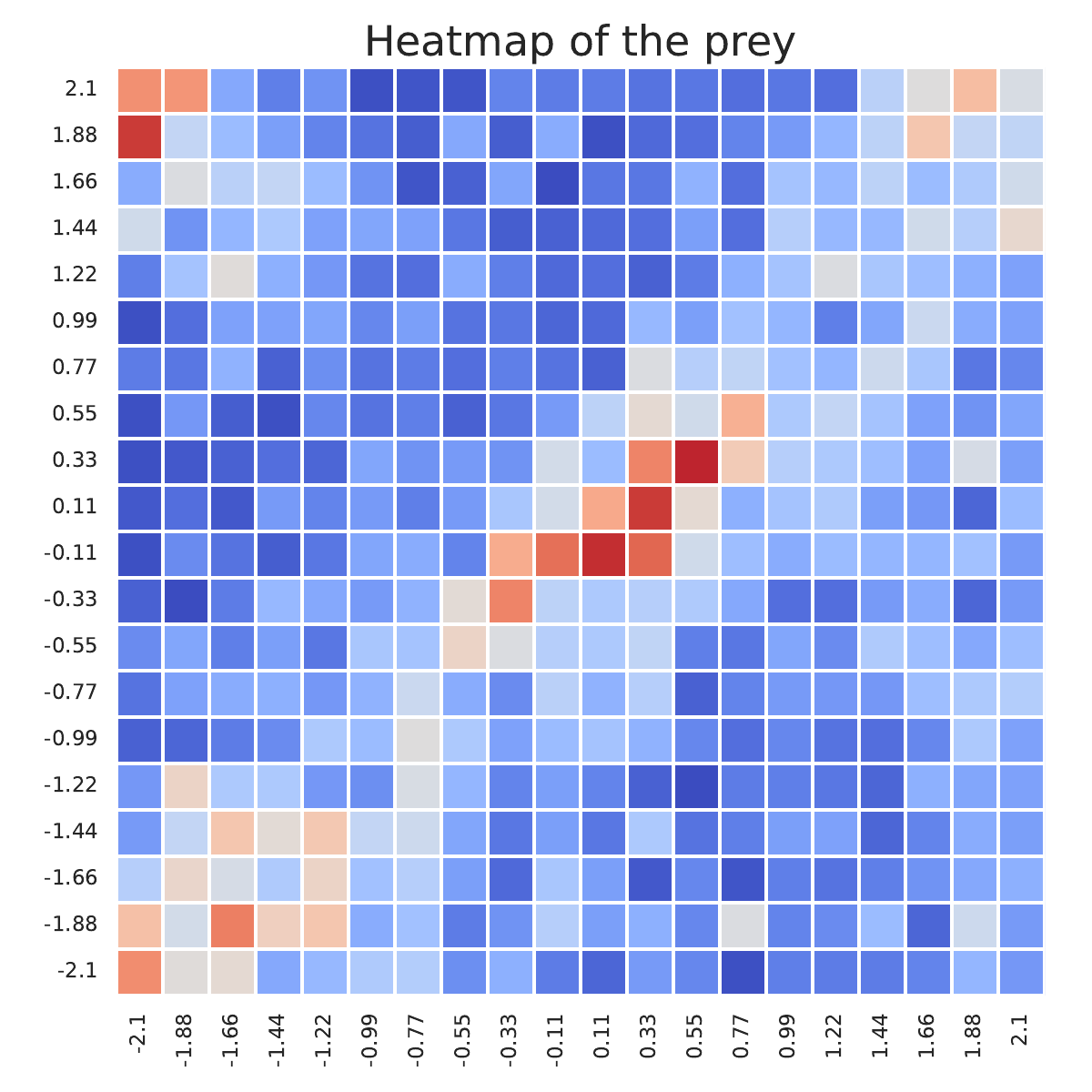}
}
\subfigure[8M: Edges.]
{
    \includegraphics[width=0.25\textwidth]{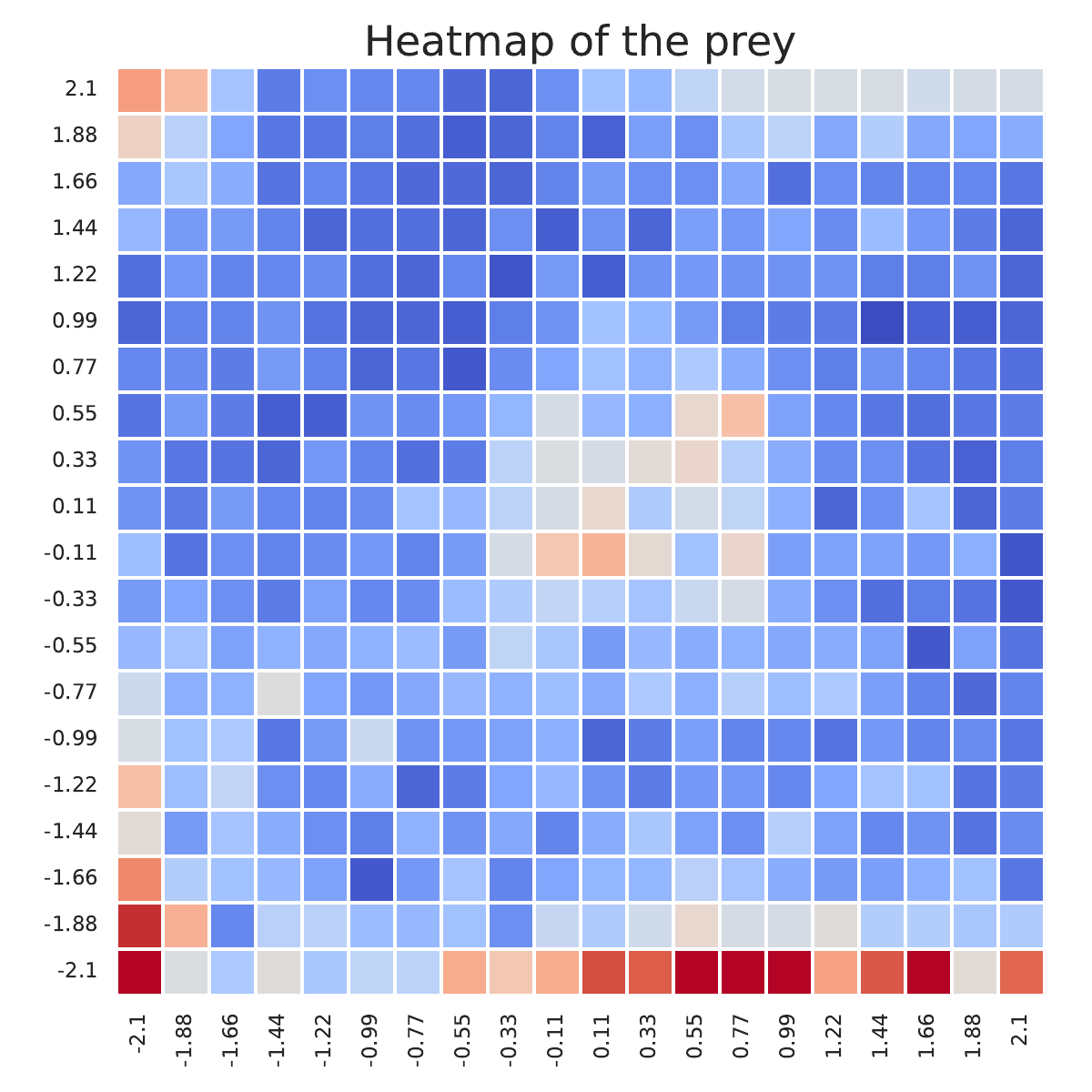}
}
\subfigure[15M: Corners.]
{
    \includegraphics[width=0.25\textwidth]{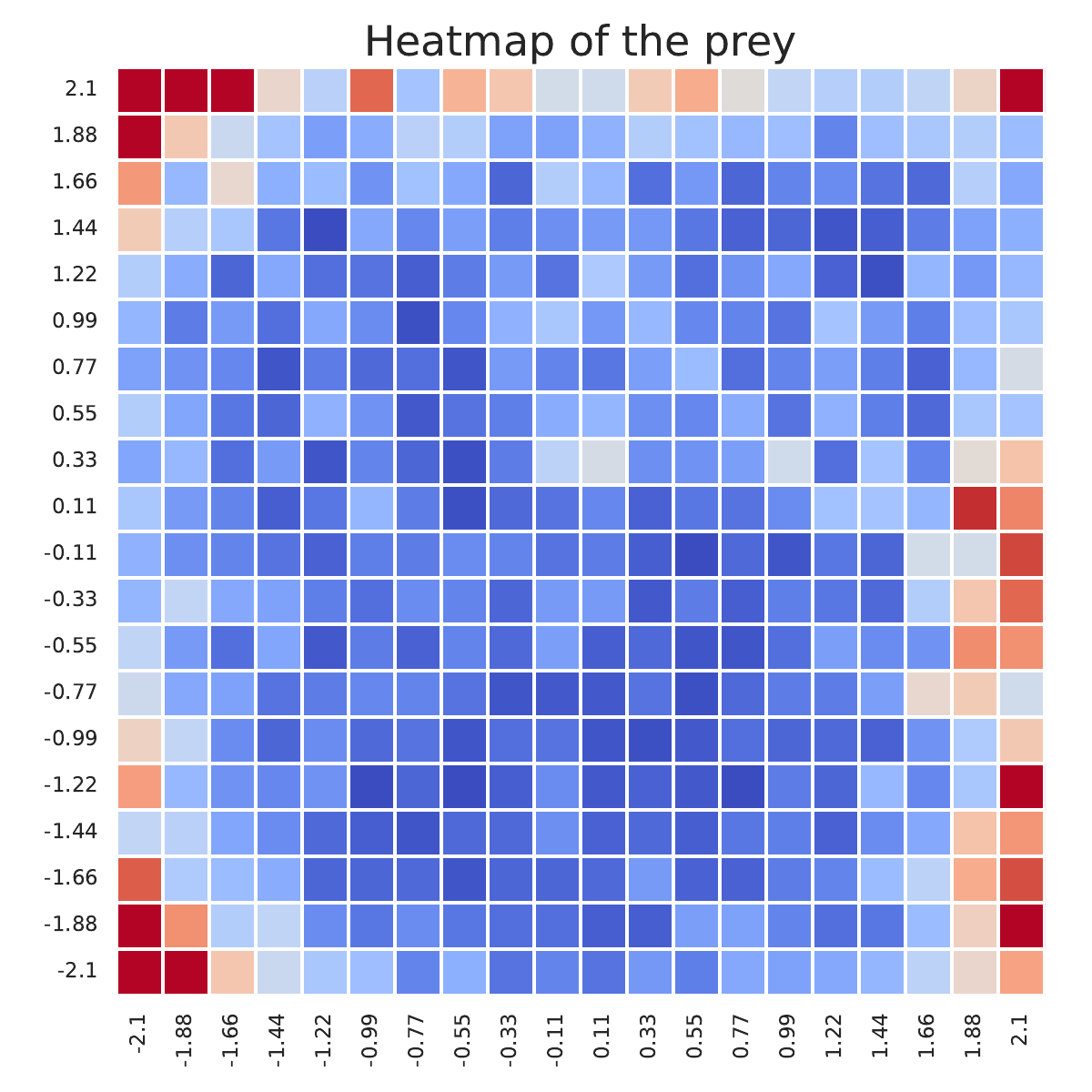}
}
\vspace{-3mm}
\caption{Visualization of the prey's initial position heatmap generated by {\name} at different training timesteps.}
\label{fig:app:sacl_hm}
\end{figure*}

\begin{figure*}[t]
\centering
\subfigure[Buffer size.]
{
    \includegraphics[width=0.23\linewidth]{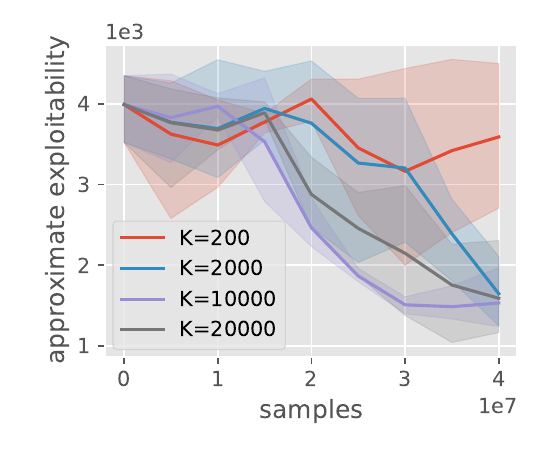}
    \label{fig:app:buffer_size}
}
\subfigure[Sample probability.]
{
    \includegraphics[width=0.23\linewidth]{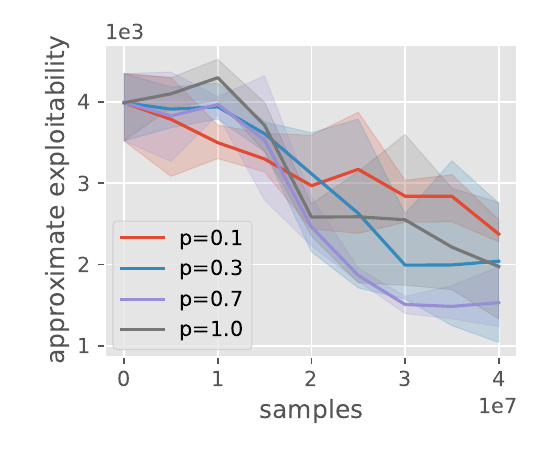}
    \label{fig:app:sample_p}
}
\subfigure[Ensemble size.]
{
    \includegraphics[width=0.23\linewidth]{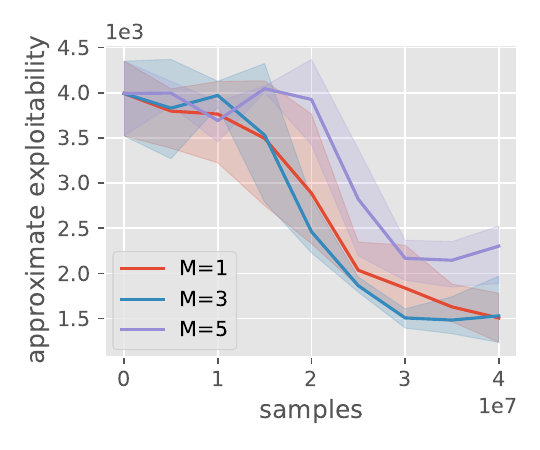}
    \label{fig:app:ensemble_size}
}
\subfigure[Weight of value difference.]
{
    \includegraphics[width=0.23\linewidth]{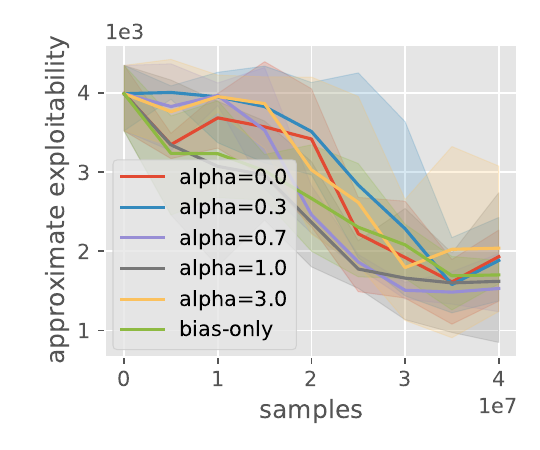}
    \label{fig:app:alpha}
}
\vspace{-3mm}
\caption{Ablation studies in MPE hard.}
\label{fig:app:ablation1}
\end{figure*}

\begin{figure*}[t]
\centering
\includegraphics[width=0.7\textwidth]{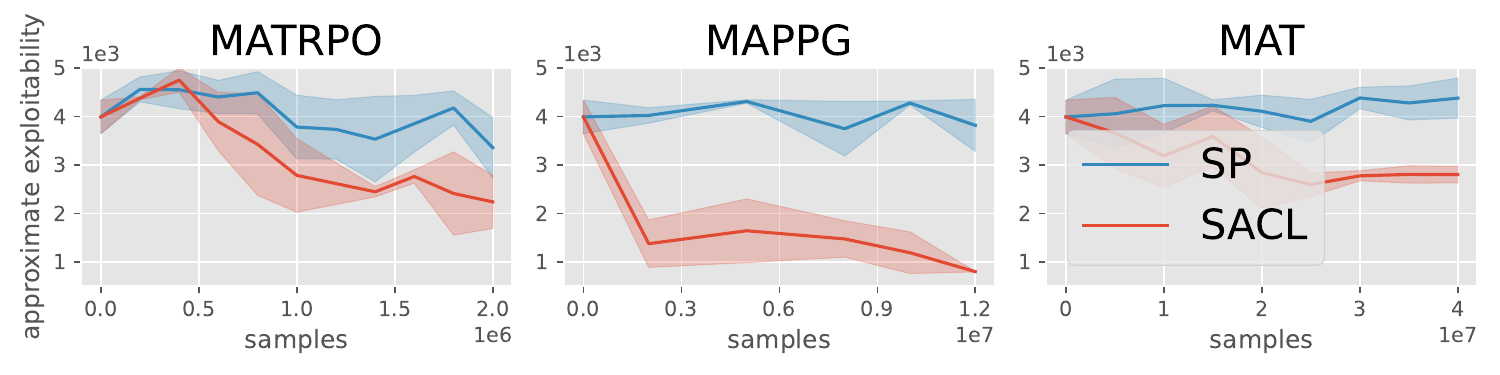}
\vspace{-2mm}
\caption{SACL outperforms self-play (SP) in all three MARL algorithms.}
\label{fig:app:marl}
\end{figure*}

\subsection{Evaluation Details}\label{sec:app:exploitability}

\textbf{Exploitability.}

In zero-sum games, because the performance of one player's policy depends on the other player's policy, the return curve throughout training is no longer a suitable evaluation method.
One way to compare the performance of different policies is to use cross-play, which uses a tournament-style match between any two policies and records the results in a payoff matrix.
However, due to the non-transitivity of many zero-sum games~\citep {balduzzi2019open}, winning other policies does not necessarily mean being close to NE policies. Hence, a better way to evaluate the performance of policies is to use exploitability.
Given a pair of policies $(\pi_1, \pi_2)$, the exploitability is defined as
\begin{small}
{
\begin{align}
    \mathrm{exploitability}(\pi_1, \pi_2) = \sum_{i=1}^2 \max_{\pi_i'} \mathbb{E}\left[V_i^{(\pi_i', \pi_{-i})}(s^0)\right].
\end{align}
}
\end{small}
Exploitability can be roughly interpreted as the ``distance'' to the joint NE policy.
In complex environments like the ones we use, the exact exploitability cannot be calculated because we cannot traverse the policy space to find $\pi_i'$ that maximizes the value. We compute the approximate exploitability by training an approximate best response $\tilde{\pi}_i'$ of the fixed policy $\pi_i$ using MAPPO. A BR is trained for 200M samples in MPE and 400M in GRF. The lower the exploitability, the better the algorithm. We use the checkpoints of an algorithm's policy trained with different numbers of environment steps to estimate the exploitability. Specifically, we run {\name} in MPE and save a policy checkpoint when the agent has consumed 0M, 5M, 10M, 15M, ..., and 40M environment samples. Then, we keep each checkpoint fixed and train an adversarial policy to be the best response of the fixed policy to estimate the exploitability. Then, we get an exploitability curve of {\name} over samples. Finally, we repeat this procedure for two more seeds, average the results, and plot the std error. For a single algorithm, we trained $9 \times 3 = 27$ (checkpoints $\times$ seeds) best-response policies to plot one curve in the exploitability graph.

\textbf{Cross-play.}
We evaluate {\name} and other baselines by cross-play, which uses a head-to-head match between any two policies and records the results in a payoff matrix. In MPE, the element of the payoff matrix represents the episodic reward of the predators, and in GRF, it represents the win rate of the red team. More specifically, we train three seeds for each algorithm and match three models of one algorithm against the three models of the opponent algorithm, i.e., we get $3\times3=9$ competitions between any two algorithms and report the average results and the std error. For example, in MPE, we use three different predators of {\name} to compete with three different preys of SP to get the episode predator reward. We can evaluate the performance of the predator using the elements of a row and evaluate the performance of the prey using the elements of a column. We use the first row to represent the predator of {\name}, then a larger value in this row than other rows means that the predator of {\name} is better than other algorithms. We use the first column to represent the prey of {\name}, then a smaller value in this column than in other columns means that the prey of {\name} is better than other algorithms.



\begin{figure*}[t]
\centering
\subfigure[FPS.]
{
    \includegraphics[width=0.25\textwidth]{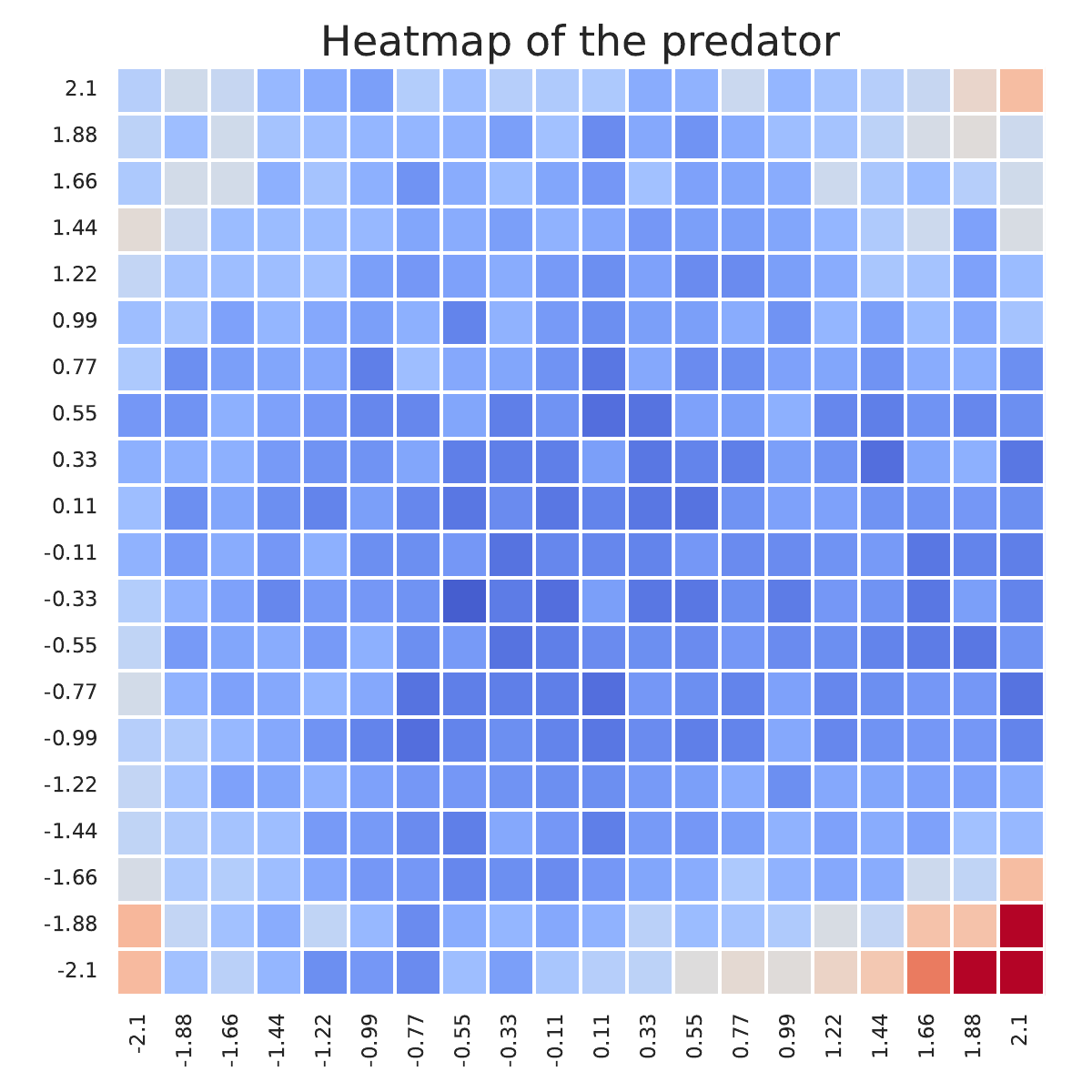}
    \label{fig:app:fps}
}
\subfigure[Greedy.]
{
    \includegraphics[width=0.25\textwidth]{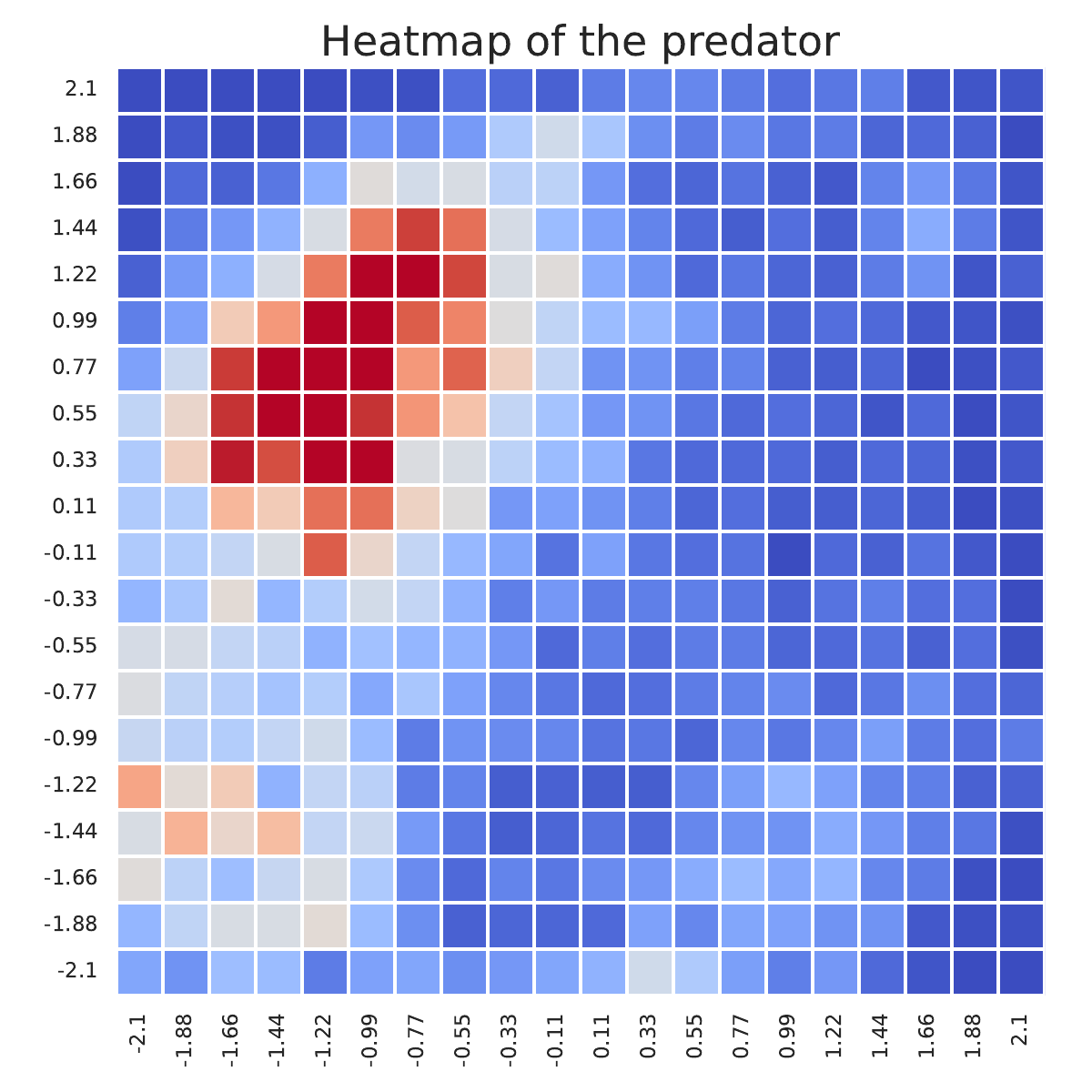}
    \label{fig:app:greedy}
}
\subfigure[Random.]
{
    \includegraphics[width=0.25\textwidth]{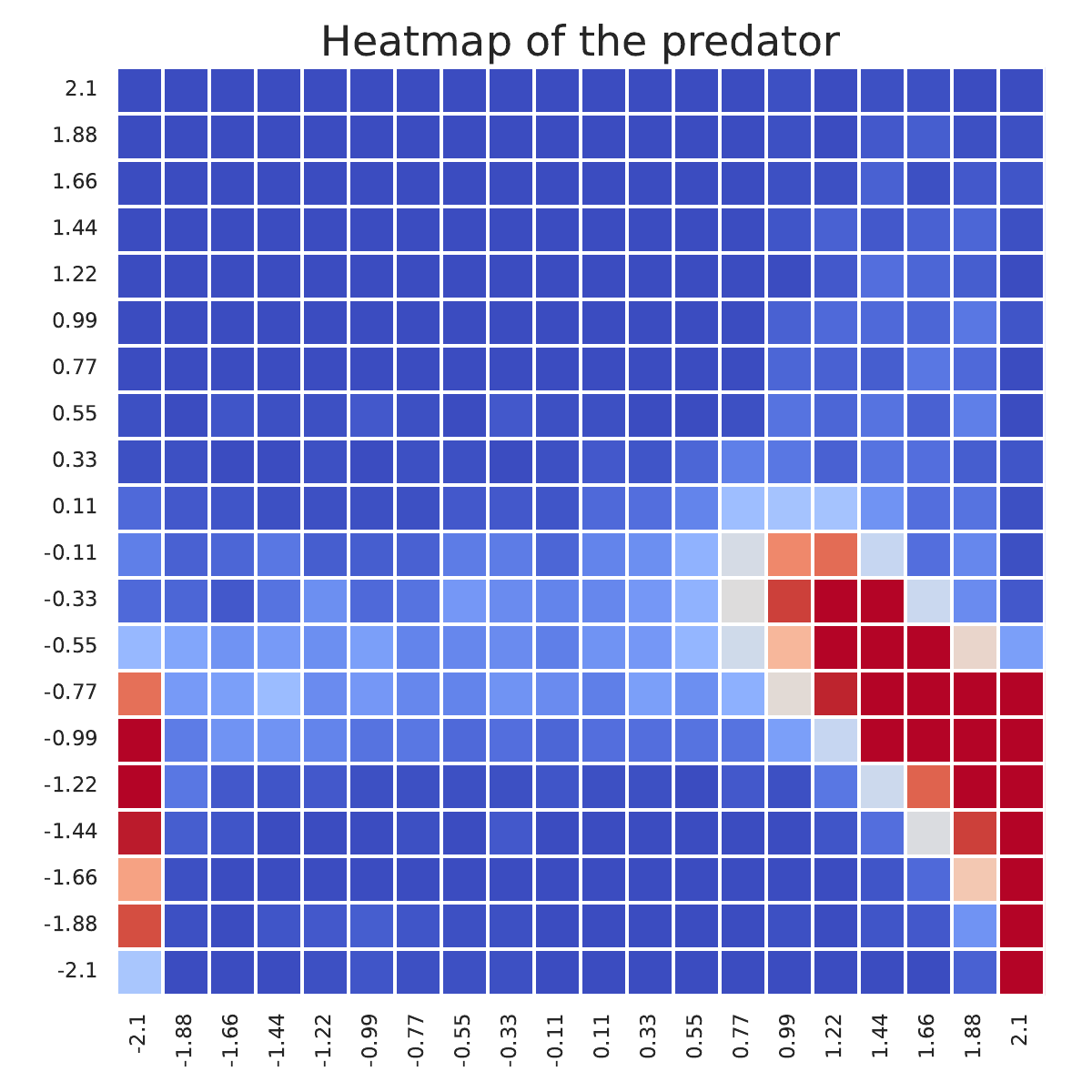}
    \label{fig:app:random}
}
\subfigure[FPS.]
{
    \includegraphics[width=0.25\textwidth]{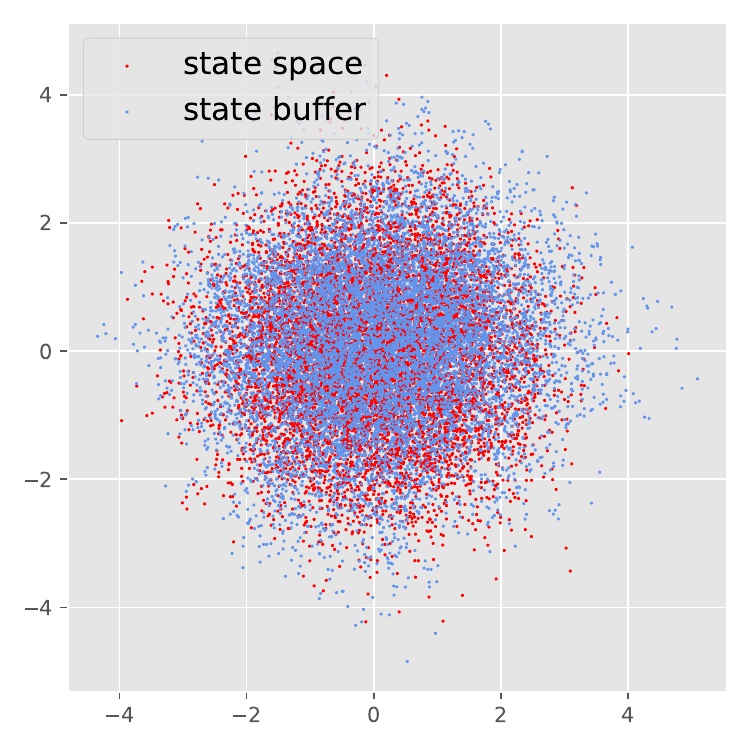}
    \label{fig:app:fps_pca}
}
\subfigure[Greedy.]
{
    \includegraphics[width=0.25\textwidth]{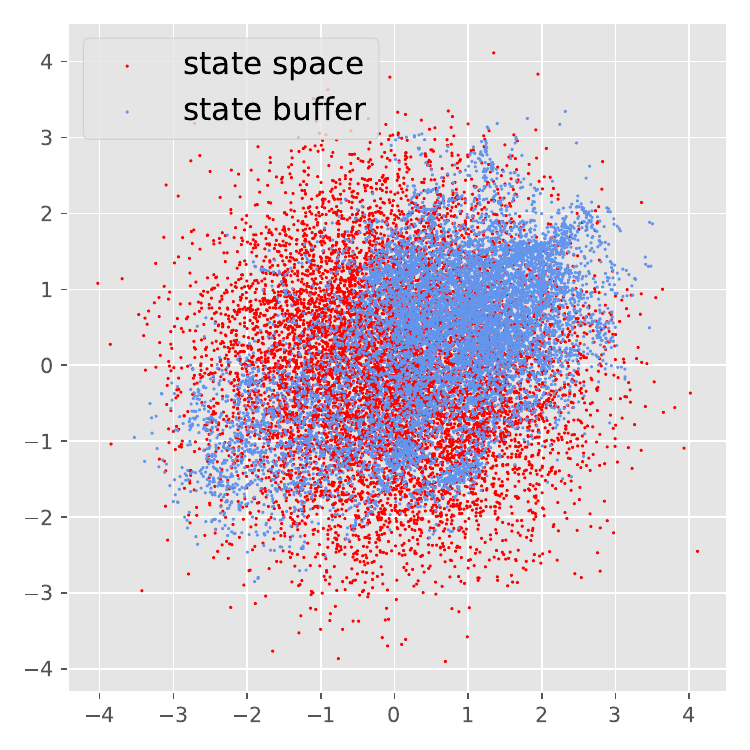}
    \label{fig:app:greedy_pca}
}
\subfigure[Random.]
{
    \includegraphics[width=0.25\textwidth]{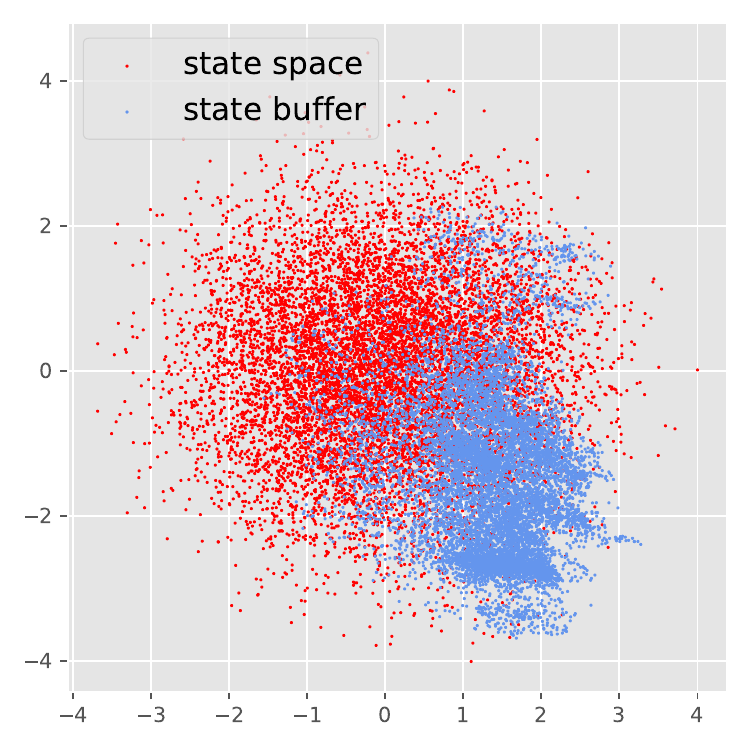}
    \label{fig:app:random_pca}
}
\vspace{-3mm}
\caption{Visualization of the state buffer and projection of tasks in the state buffer to 2-dimension by principal component analysis generated by three update methods.}
\label{fig:app:update_method}
\vspace{-2mm}
\end{figure*}

\begin{figure*}[t]
\centering
    \subfigure[Pass and shoot.]
{
    \includegraphics[width=0.25\textwidth]{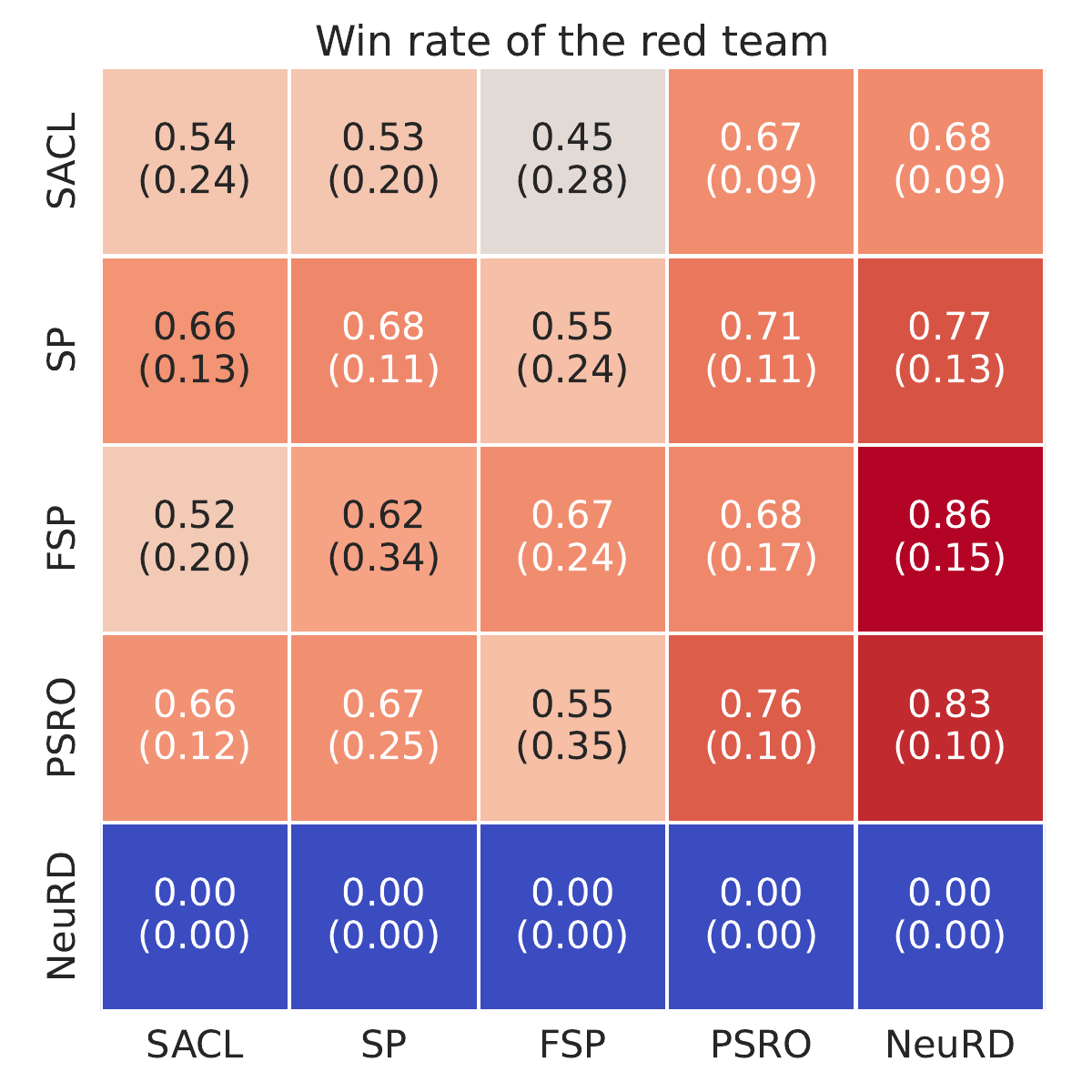}
    \label{fig:app:cross_ps}
}
\subfigure[Run, pass and shoot.]
{
    \includegraphics[width=0.25\textwidth]{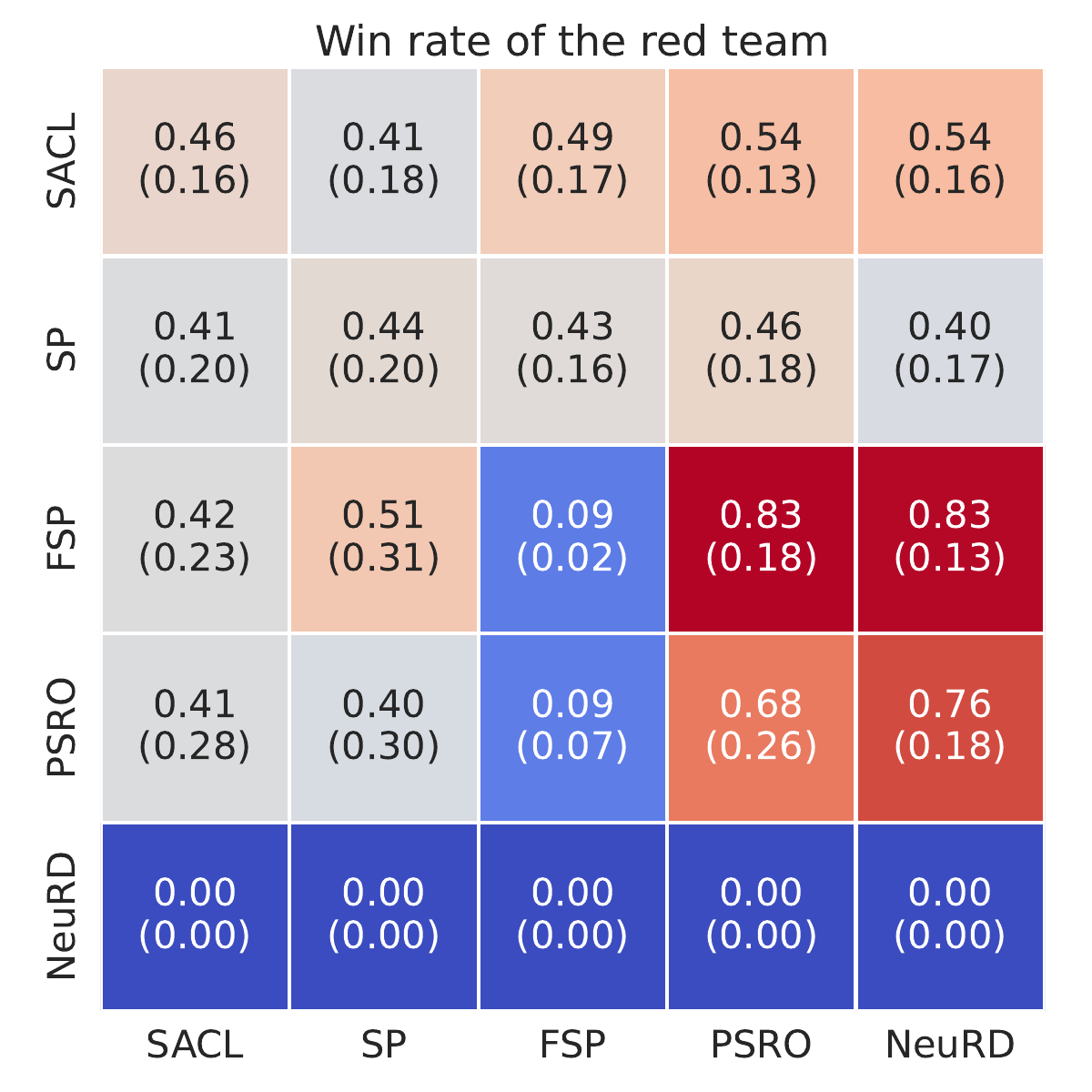}
    \label{fig:app:cross_rps}
}
\subfigure[3 vs 1 with keeper.]
{
    \includegraphics[width=0.25\textwidth]{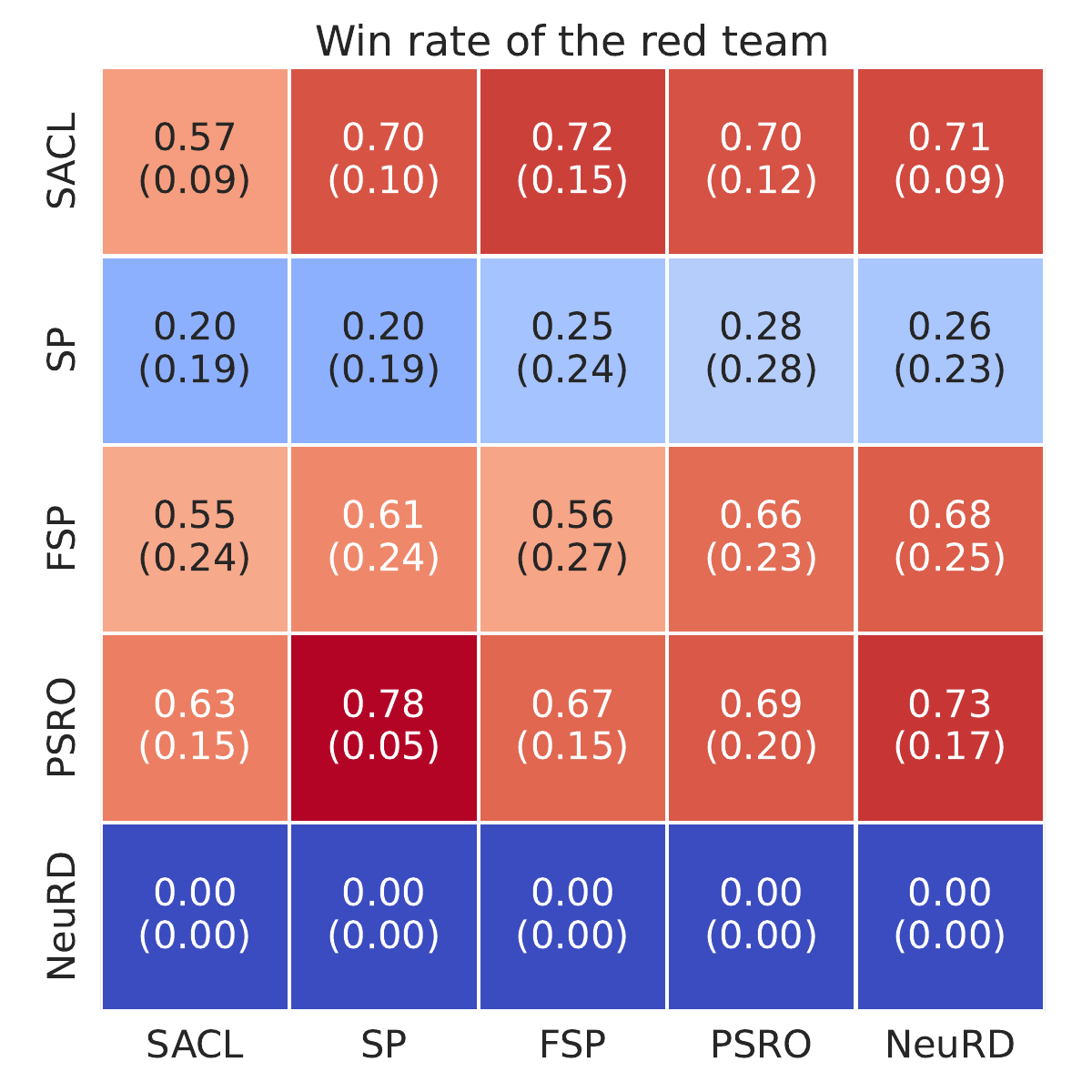}
    \label{fig:app:cross_3vs1}
}
\caption{Cross-play results in GRF.}
\label{fig:app:cross_grf}
\end{figure*}

\textbf{Four rounds of emergent strategies in HnS.} As shown in Fig.~\ref{fig:exp:hns_appendix}, we use three inflection points to evaluate the sample required to produce the first three stages.
More specifically, the \emph{Running and Chasing} phase ends when the hider's reward decreases to the lowest value of about $-20$.
When the hider's reward begins to increase, the \emph{Fort-Building} phase begins and continues until the hider's reward reaches a local maximum of about $30$.
Then, the agents move to the \emph{Ramp-Defense} phase until the hider's reward reaches a local minimum and begins the final \emph{Ramp-Use} stage.
We choose the point when the hider's episode reward reaches $40$ as the end of the final stage.

\section{Additional Experiment Results}\label{sec:app:exp}

\subsection{Multi-Agent Particle Environment}\label{sec:app:mpe}
\textbf{Cross-play.}
The results of cross-play at $40M$ in MPE and MPE hard are shown in Fig.~\ref{fig:app:cross_easy} and Fig.~\ref{fig:app:cross_hard}. 
In MPE and MPE hard, the predator and prey of {\name} beat all baselines. 
For example, let x be the row x and y represent the column y of the payoff matrix. We compare the predator of  {\name} with FSP using rows 1 and 3 and find that the elements of row 1 are larger than the elements of row 3, i.e., the predator of {\name} is better than FSP. The elements of column 1 are smaller than the elements of column 3, which means the prey of {\name} is better than FSP. The prey trained by {\name} swerves to avoid the predators when the predators surround him, and the predators learn to capture the prey in the two environments. SP is comparable with {\name} in MPE, but in the hard setting, SP does not converge to the NE policy due to the large initial distance between predator and prey. 
FSP also performs worse than {\name} in the hard setting for the same reason as SP. For PSRO, it is even difficult to obtain the best response corresponding to the prey of random policy in MPE hard because the initial distance between prey and predator is too far. NeuRD performs poorly in both MPE and MPE hard because NeuRD's update rules cause drastic policy changes and erratic convergence.

\textbf{Visualization of subgame curriculum.}
To show the effectiveness of {\name}, we visualize the change of the prey’s initial position heatmap produced by {\name} in the MPE hard in Fig.~\ref{fig:app:sacl_hm} and find that it starts from the center and moves to the edges and corners which means that we train from easier subgames and gradually move to harder ones. 



\begin{figure}[t]
    \centering
    \includegraphics[width=0.7\linewidth]{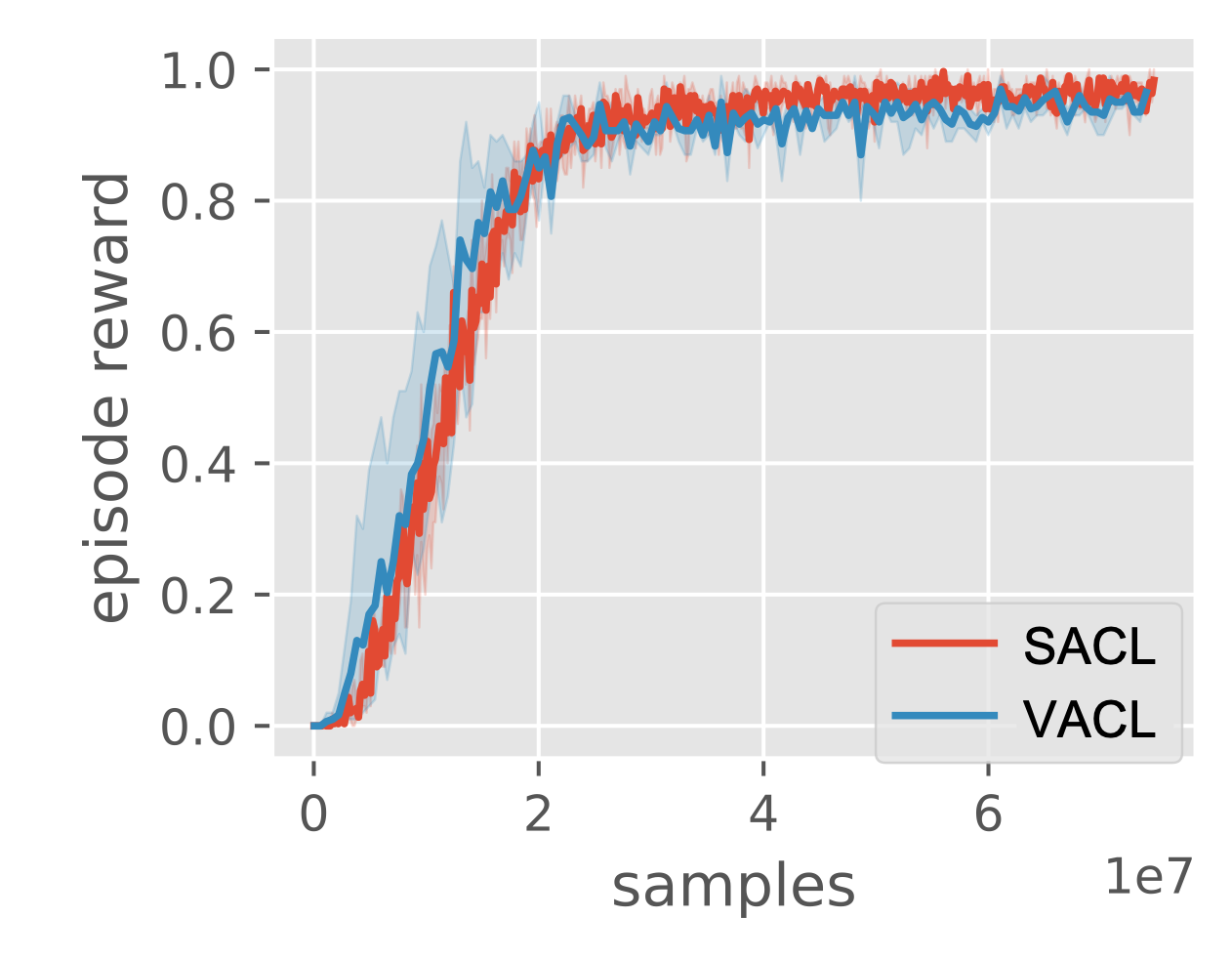}
    \caption{Seeker's episode reward in the goal-reaching \emph{Ramp-Use} task. {\name} is comparable to a strong baseline VACL, which is specialized for goal-oriented problems.}
    \label{fig:app:coop_game}
\end{figure}

\textbf{Buffer size.} 
As shown in Fig.~\ref{fig:app:buffer_size}, the buffer capacity $K$ must be large enough. When the buffer is too small, the states in the buffer cannot approximate the state space. When the buffer is too large, FPS consumes much time. So we finally choose $K=10000$.

\textbf{Subgame sample probability.}
As shown in Fig.~\ref{fig:app:sample_p}, we need more samples from the subgame buffer than uniform sampling in the training batch, and uniform sampling from the state space ensures global exploration. When $p$ is too small, {\name} degenerates into SP, resulting in poor performance. When $p=1$, the lack of global exploration also leads to poor performance. Finally, we choose $p=0.7$.

\textbf{Ensemble size.}
As shown in Fig.~\ref{fig:app:ensemble_size}, we can train an ensemble of value functions for each player to improve the estimation. Excessive ensemble size requires much memory and training time. So we finally choose $M=3$.

\textbf{Weight of value difference.}
As shown in Fig.~\ref{fig:app:alpha}, our algorithm is insensitive to the weight of the value difference $\alpha$. Empirically, we prefer $\alpha$ less than 1. We finally choose $\alpha=0.7$ in MPE, MPE hard, and GRF, $\alpha=1.0$ in Hns.

\textbf{Different MARL backbones.}
We also conduct experiments in MPE hard with three more MARL algorithms including MAT~\citep{wen2022multi}, MATRPO, and MAPPG which are extensions of Trust Region Policy Optimization (TRPO)~\citep{schulman2015trust} and Phasic Policy Gradient (PPG)~\citep{cobbe2021phasic} based on the centralized training and decentralized execution paradigm. The results in Fig.~\ref{fig:app:marl} show that SACL accelerates the learning process of all three algorithms.

\textbf{Buffer update method.}
We further visualize the state distribution in the buffer generated by different update methods in Fig.~\ref{fig:app:update_method}. Fig.~\ref{fig:app:update_method}(a-c) shows the heatmaps of the predators' position. Fig.~\ref{fig:app:update_method}(d-f) run PCA on the full state space and show the projection of the states in the buffer to the two-dimensional space. The results show that if we randomly select states or greedily select states with high weights, the states in the buffer can become very concentrated and can't approximate the whole state space.

\subsection{Google Research Football}\label{sec:app:grf}
The results of cross-play in \textit{pass and shoot}, \textit{run, pass and shoot}, and \textit{3 vs 1 with keeper} are shown in Fig.~\ref{fig:app:cross_grf}. 
In the three scenarios, {\name} is comparable to FSP and PSRO and better than SP and NeuRD. For example, let x be the row x, and y be the column y of the payoff matrix. In \textit{3 vs 1 with keeper}, the elements of row 1 are larger than the elements of row 2, which means the attackers of {\name} are better than SP. The elements of column 1 are comparable with the elements of column 2, i.e., the prey of {\name} is comparable with SP. It is worth mentioning that in \textit{run, pass and shoot}, FSP and PSRO attackers have a higher win rate than {\name} against PSRO and NeuRD defenders. This is because PSRO and NeuRD defenders have a bad defensive policy, and FSP and PSRO attackers have their counter policy. However, Table 1 in the main text shows that the exploitability of {\name} is lower than others. This is because zero-sum games are non-transitive. For example, in rock-paper-scissors, it doesn't mean that rock is better than paper just because rock beats scissors and scissors beat paper. Thus, a high return against a single policy does not mean that it is close to the NE policy, and the comparable result in cross-play does not contradict the exploitability result. In general, exploitability is a better measure of policy performance and is used in many papers.

We also visualize the behavior of different methods to show that {\name} learns more complex policies than others and is closer to the NE policies. For example, in \textit{3 vs 1 with keeper}, the NE policy is that the left players shoot from the top, middle, and bottom with equal probability. {\name} learns to shoot from the top and the middle, while FSP and PSRO only shoot from the bottom.

\subsection{Hide-and-seek}
Although {\name} is derived for zero-sum games, it is also applicable to more general settings such as goal-conditioned problems. We consider the \emph{Ramp-Use} task proposed in VACL~\citep{chen2021variational}, where the seeker aims to get into the lower-right quadrant (with no door opening), which is only possible by using a ramp. 
We adopt the same prior knowledge of ``easy tasks'' used in VACL to initialize the state buffer $\mathcal{M}$ and achieve comparable sample efficiency with VACL, one of the strongest ACL algorithms for goal-conditioned RL. The result is shown in Figure~\ref{fig:app:coop_game}.

\section{Future Directions to Extend {\name}}\label{app:future}

\subsection{For Partially-Observable Markov Games}
{\name} can be directly used in fully observable Markov games where the states contain all information about the game and are observable to all agents. For partially observable Markov games, though some of the information is hidden from the agents, the states still contain all information of the game, and it is also possible to run {\name} in these games. An important part is to deal with the distribution of hidden information. A way to do that is to replace states in prioritized sampling with infosets, i.e., sets of states that are indistinguishable from agents, and maintain the distribution of states within each infoset. In each episode, we first use prioritized sampling to select an infoset and then sample a state from the infoset to generate the subgame. In this way, we keep the distribution of hidden information and also build a subgame curriculum to accelerate training.

\subsection{For General-Sum Games}
{\name} consists of three components: the subgame curriculum learning framework, the sampling metric, and the particle-based sampler. The framework and the sampler can be applied to general-sum games because they don't require the zero-sum property. The only part to change is the metric.
Unfortunately, to the best of our knowledge, there is no clear metric to measure the subgames' learning progress in general-sum games. A possible way is to start from Eq.~\ref{eq:oracle} to derive another metric, which is a possible direction for future work.

\end{document}